\pdfoutput=1

\documentclass[11pt]{article}
\usepackage[normalem]{ulem}
\usepackage[]{EMNLP2022}

\usepackage{times}
\usepackage{latexsym}
\usepackage{todonotes}

\usepackage[T1]{fontenc}

\usepackage[utf8]{inputenc}

\usepackage{amsmath,amsfonts,bm}

\def\eqref#1{equation~\ref{#1}}

\def\1{\bm{1}}

\def\rh{{\textnormal{h}}}

\def\rvg{{\mathbf{g}}}
\def\rvh{{\mathbf{h}}}

\def\rvx{{\mathbf{x}}}

\def\rvz{{\mathbf{z}}}

\def\vb{{\bm{b}}}

\def\vx{{\bm{x}}}
\def\vy{{\bm{y}}}
\def\vz{{\bm{z}}}

\def\mA{{\bm{A}}}

\def\mD{{\bm{D}}}

\def\mH{{\bm{H}}}
\def\mI{{\bm{I}}}

\def\mL{{\bm{L}}}

\def\mU{{\bm{U}}}

\def\mW{{\bm{W}}}
\def\mX{{\bm{X}}}

\def\mZ{{\bm{Z}}}

\def\mLambda{{\bm{\Lambda}}}

\DeclareMathAlphabet{\mathsfit}{\encodingdefault}{\sfdefault}{m}{sl}
\SetMathAlphabet{\mathsfit}{bold}{\encodingdefault}{\sfdefault}{bx}{n}

\usepackage{microtype}
\usepackage{graphicx}
\usepackage{subfigure}
\usepackage{booktabs}
\usepackage{amsmath}
\usepackage{amssymb}
\usepackage{mathtools}
\usepackage{amsthm}

\usepackage[capitalize,noabbrev]{cleveref}

\theoremstyle{plain}
\newtheorem{theorem}{Theorem}[section]

\theoremstyle{definition}
\newtheorem{definition}[theorem]{Definition}

\theoremstyle{remark}

\usepackage{todonotes}

\usepackage{amsbsy}
\usepackage{times}
\usepackage{latexsym}
\usepackage{amsfonts}
\usepackage{dsfont}
\usepackage{pgfplots}
\pgfplotsset{compat=1.17}
\usepackage{diagbox}
\usepackage{xspace}
\usepackage{multirow}
\usepackage{color}
\usepackage{wrapfig}
\usepackage[font=small,labelfont=bf]{caption}
\usepackage{comment}
\usepackage{tikz}
\usetikzlibrary{intersections}
\newcommand{\MI}{{\rm MI}}
\newcommand{\GCS}{{\rm GCS}}
\newcommand{\GFT}{{\rm GFT}}
\newcommand{\RGFT}{{\rm RGFT}}
\newcommand{\MLP}{{\rm MLP}}
\usepackage{scalefnt}

\usepackage{inconsolata}

\title{
What Has Been Enhanced in my Knowledge-Enhanced Language Model?
}

\author{Yifan Hou$^{1}$, Guoji Fu$^{2}$, Mrinmaya Sachan$^{1}$ \\
  $^{1}$ETH Z\"urich, \quad  $^{2}$Southern University of Science and Technology \\
  $^{1}$\texttt{\{yifan.hou, mrinmaya.sachan\}@inf.ethz.ch}, \quad $^{2}$\texttt{11749236@mail.sustech.edu.cn}}

\begin{document}
\maketitle
\begin{abstract}
A number of knowledge integration (KI) methods have recently been proposed to incorporate external knowledge into pretrained language models (LMs).
Even though knowledge-enhanced LMs outperform base LMs on knowledge-intensive tasks, the inner-workings of these KI methods are not well-understood. 
For instance, it is unclear which knowledge is effectively integrated into knowledge-enhanced LMs and which is not; and if such integration leads to catastrophic forgetting of already learned knowledge. 
We show that existing model interpretation methods such as linear probes and prompts have some key limitations in answering these questions.
We revisit KI from an information-theoretic view and propose a new theoretically sound probe called \textit{Graph Convolution Simulator} (GCS) for KI interpretation.
GCS uses graph attention on the corresponding knowledge graph for interpretation.
In our experiments we verify that GCS can provide reasonable interpretation results for two well-known knowledge-enhanced LMs: ERNIE and K-Adapter.
We also find that only a marginal amount of knowledge is successfully integrated in these models, and simply increasing the size of the KI corpus may not lead to better knowledge-enhanced LMs.\footnote{Our code, demo, and instructions of the usage can be found in \href{https://github.com/yifan-h/GCS_KI}{https://github.com/yifan-h/GCS\_KI}}
\end{abstract}

\section{Introduction}
Pretrained language models (LMs) have become the backbone of NLP. Recent work has shown that linguistic knowledge is captured quite well in LMs ~\citep{liu-etal-2019-linguistic,jawahar-etal-2019-bert}.
However, 
LMs are much worse at capturing factual knowledge about the world~\citep{petroni-etal-2019-language, wang-etal-2021-kepler}.
This has led to the development of a number of knowledge integration (KI) methods which integrate external knowledge from knowledge graphs (KGs) into LMs, leading to knowledge-enhanced language models such as \textit{ERNIE} ~\citep{zhang-etal-2019-ernie} and \textit{K-Adapters} \citep{wang-etal-2021-k}.

Even though enhanced LMs perform better on knowledge-intensive tasks, there is little understanding of where this improvement comes from.
Which factual knowledge is successfully integrated into the LM, and which kind of knowledge is not, is not well-understood.
%
As new knowledge is integrated in LMs, old knowledge could be \textit{catastrophically forgotten}~\citep{forgetting_kirkpatrick16}. KI could also lead to a situation called \textit{catastrophic remembering}~\citep{crcf_kaushik21}, where the old knowledge could prevent the integration of new knowledge. Our understanding about these issues is also limited.

\begin{figure*}[!htbp]
	\centering
	\includegraphics[scale=.135]{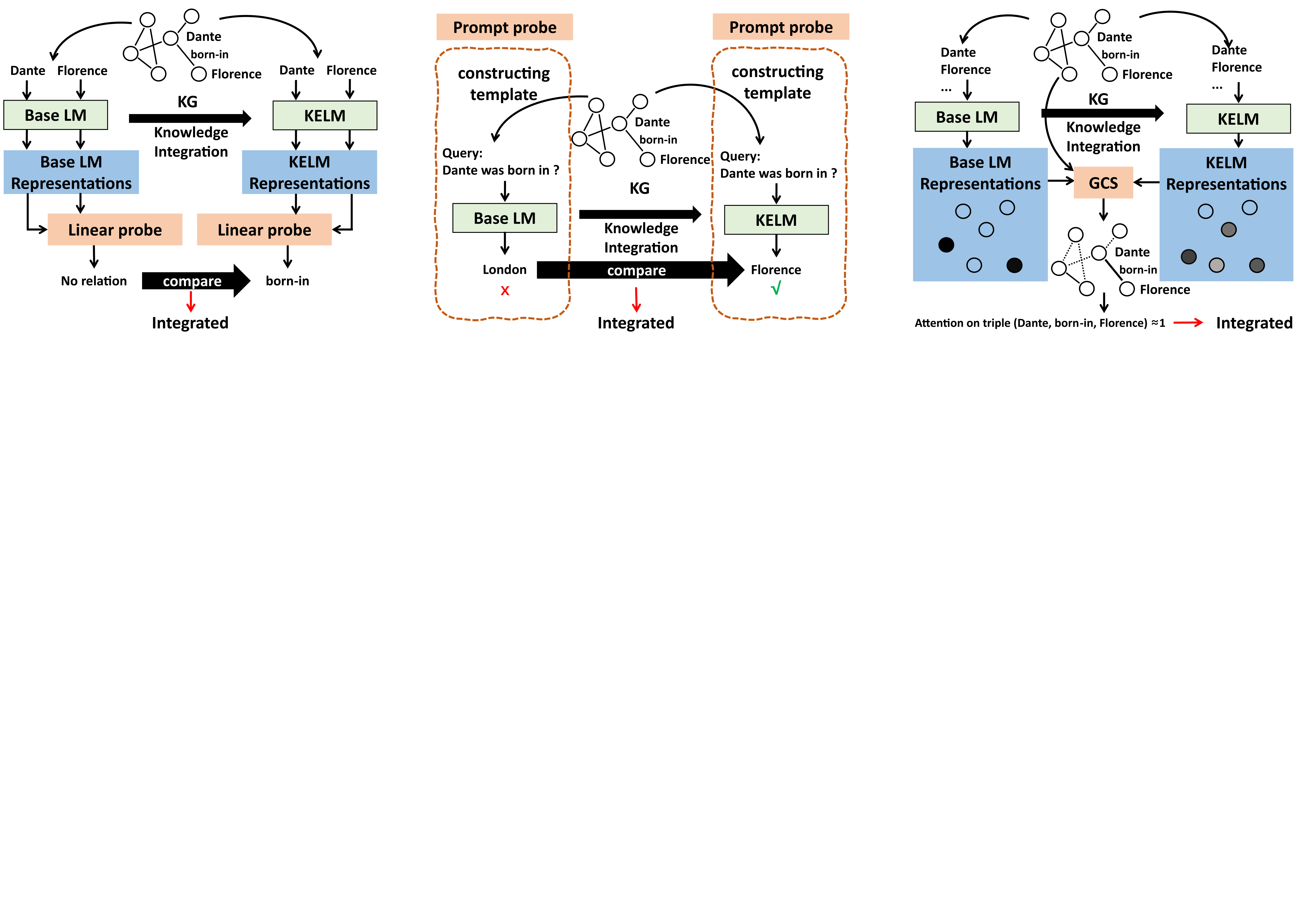}
	\vspace{-.2cm}
	\caption{Adaptations of two existing probe methods and GCS to interpret knowledge integration in language models. Probes from left to right are: linear probe, prompt-based probe, and GCS. We can find that only GCS does not have the \textit{compare} operation, which can avoid introducing extra noise of interpretation results.}
	\vspace{-.2cm}
	\label{fig:compare}
\end{figure*}

An intuitive way to understand the KI process could be to probe for factual knowledge in the base LM and the enhanced LM respectively, and compare their results for interpretation.
However, we find that because of the high variance of classifier-based probes for KG prediction~\cite{DBLP:conf/pkdd/MenonE11,DBLP:conf/aaai/LiCT16} and the significant human effort required to verbalize knowledge graph facts into templates~\cite{petroni-etal-2019-language,shin-etal-2020-autoprompt}, we cannot easily extend existing probe methods to interpret knowledge-integration in LMs (\S\ref{section:uns}).

In this paper, we revisit the KI process and formulate it with an information-theoretic view (\S\ref{section:method}). 
We measure the factual knowledge encoded in the LM by the mutual information (MI) between the model representations and the KG~\citep{hou-sachan-2021-birds}. 
Now, the integration and forgetting of factual knowledge can be measured using the difference in MI between the KG and the base LM and the MI between the KG and the enhanced LM. 
Based on this idea, we theoretically derive a transformation composed of graph convolutions to simulate and interpret this change in MI, leading to our proposed probe, Graph Convolution Simulator (GCS) (\S\ref{section_gcs}). The interpretation mechanism in GCS is quite different from existing probe methods as shown in Figure~\ref{fig:compare}. 
GCS uses a graph attention layer on the KG to simulate the MI change. Then, it interprets KI from the information flow on the KG using graph attention coefficients.

In our experiments (\S\ref{section:exp}), we verify that GCS can provide reasonable KI interpretation results. 
We show that: (a) results of GCS have significantly smaller variance compared to the linear probe, (b) dropping knowledge identified as \textit{non-integrated} by GCS does not affect enhanced LMs' performance, and (c) enhanced LMs perform better on samples that include knowledge identified as \textit{integrated} by GCS.
In particular, we use GCS to understand two well-known knowledge-enhanced LMs: \textit{ERNIE}~\citep{zhang-etal-2019-ernie} and \textit{K-Adapter}~\citep{wang-etal-2021-k}. 
Our findings are listed as follows.
\begin{itemize}
    \item Both of them only integrate little new knowledge (i.e., less than 30\% knowledge triples are successfully integrated). ERNIE is better at integrating triples with high degree in KGs, while K-Adapter integrates triples with low-degree entities well.
    
    \item In our qualitative study, we find that enhanced LMs do not integrate numbers and temporal knowledge well, where only less than 0.01\% triples are successfully integrated.
    
    \item Finally, we find that there is no positive relationship between KI corpus size and KI quality. This suggests that merely building larger corpus would not be enough, highlighting the need for more fundamental advances in KI.
\end{itemize}


\section{Preliminaries}
\paragraph{KI methods.} There are several approaches for KI. 
KI in LMs can be implemented by aligning phrases in text to entities ~\citep{peters-etal-2019-knowledge,zhang-etal-2019-ernie} or triples ~\citep{kbert_liu20,wang-etal-2021-k} and incorporating corresponding entity or the triple embeddings in the LM.
KI methods also include modifications to the Transformer architecture~\citep{peters-etal-2019-knowledge,kbert_liu20}, 
verbalizing knowledge triples and using data augmentation for finetuning~\citep{agarwal-etal-2021-knowledge}, and designing objective functions that predict the factual knowledge~\citep{kgbert_yao19,wang-etal-2021-k}.

\paragraph{Knowledge graphs.} 
We assume that factual knowledge for integration can be formulated as a KG $\mathcal{G} = ( \mathcal{V}, \mathcal{E})$, where nodes $v_i \in \mathcal{V}$ represent entities, and edges in $\mathcal{E}$ represent relations between them. 
Let $\mathcal{N}_{v_i}$ denote the set of neighbors of node $v_i$, and $t_i$ denote the entity label corresponding to the node $v_i$. Further, let $\vx_i = {\rm LM}(t_i)$ denote the entity (label) representations given by a LM\footnote{We represent each entity as the average of its word(-piece) embeddings given by the LM as~\citet{hewitt-manning-2019-structural}.}, and $\mX \in \mathbb{R}^{|\mathcal{V}|\times d}$ denote a matrix formed by stacking all entity representations $\vx_i \in \mathbb{R}^d$.
In this paper, we only consider nodes and edges in the KG and ignore other side information such as relation weights, directions and labels\footnote{Note that the relation label follows the power-law distribution, and existing probe methods also can not support it well as discussed in Appendix~\ref{appendix:relation}}.

\section{Unsuitability of Existing Probes for KI}\label{section:uns}
Classifier probes~\cite{linearprobe_guillaume,hewitt-liang-2019-designing} and prompt-based probes (e.g. the LAMA probe)~\cite{petroni-etal-2019-language,shin-etal-2020-autoprompt} are typically used to test for various kinds of knowledge in LMs.
Classifier probes train simple (usually linear) classifiers to predict the linguistic property of interest and the probe accuracy is used for interpretation~\citep{ribeiro-etal-2016-trust,hewitt-manning-2019-structural}. 
However, simple classifiers are not powerful enough to make reliable predictions about large KGs~\citep{DBLP:conf/pkdd/MenonE11,DBLP:conf/aaai/LiCT16}.\footnote{Note that if a more powerful probe model is used, it would learn the task itself instead of probing the task as discussed in~\citep{hewitt-manning-2019-structural,pimentel-etal-2020-information}}
Moreover, linear probes are also unable to provide reasonable insights for LMs as they suffer from high variance. If we use them to probe two LMs and compare their results for KI interpretation, variance of interpretation would further increase. 
We provide empirical evidence later in \S\ref{exp:verify:variance}.

Prompting is another popular way to understand what factual knowledge LMs know. 
Prompts are designed to let LMs solve text infilling problems, and the prompt output is then used for interpretation~\citep{petroni-etal-2019-language}. 
However, people have to manually design templates for factual knowledge to be probed\footnote{Even if existing prompt methods could learn the template automatically~\citep{shin-etal-2020-autoprompt}, the training process could also learn the task and provide unreliable probe results~\citep{jiang-etal-2021-know,zhong-etal-2021-factual,cao-etal-2021-knowledgeable}.}, and the quality of templates is vital in the overall prompt accuracy~\citep{jiang-etal-2021-know,li-etal-2022-probing-via}. As KI methods often use large KGs for integration, 
it would be infeasible to write a large number of templates for all triples in KGs.
To address these issues, we introduce the GCS model and its theoretical motivation behind.

\section{Knowledge Integration Understanding}\label{section:method}
First, we revisit KI in LMs by formulating it in an information-theoretic way. Then, we construct transformations to simulate and interpret the KI process. Notations are summarized in Appendix~\ref{appendix:notations}.

\subsection{Knowledge Integration Definition}\label{section:method:definition}
We measure knowledge in LMs using mutual information (MI) between the knowledge and the LM representations~\citep{hou-sachan-2021-birds}.
We assume that the local graph $\mathcal{G}(v_i)$ contains all factual knowledge regarding on $v_i$, and successfully integrated knowledge should be reflected in the entity representations.
Let $\rvx$ be a random variable that takes values ranging over all possible entity representations of a LM\footnote{Here, the set of entity representations $\mX \in \mathbb{R}^{|\mathcal{V}|\times d}$ can be regarded as empirical samples from $\rvx$.}, and $\rvg$ be a random variable that ranges over all possible corresponding local structures $\mathcal{G}(v_i)$. 
Mutual information $\MI(\rvx; \rvg)$ can be used to measure the amount of information in $\rvg$ contained in $\rvx$ as 
$$\MI(\rvx; \rvg) = D_{KL}(\mathbb{P}_{\rvx\rvg}||\mathbb{P}_{\rvx} \otimes \mathbb{P}_{\rvg}),$$ which is equivalent to the Kullback-Leibler (KL) divergence between the joint distribution $\mathbb{P}_{\rvx\rvg}$ and the product of the marginal distributions $\mathbb{P}_{\rvx} \otimes \mathbb{P}_{\rvg}$. Next, we present a formal definition of KI.

\begin{definition}[{Knowledge Integration}] \label{def_ki}
    Let $\rvx$, $\rvh$ denote random variables for entity representations in the base LM and the enhanced LM, respectively.
    The KI process can be formulated as a function $f$ such that $\rvh = f(\rvx,\rvg)$. Consequently, we assume that the knowledge change during KI can be measured using MI by: $\MI(\rvx;\rvg) \rightarrow  \MI(\rvh;\rvg)$.
\end{definition} 

\begin{wrapfigure}{ro}{0.28\textwidth}
    \vspace{-.3cm}
    \centering
    \includegraphics[scale=.12]{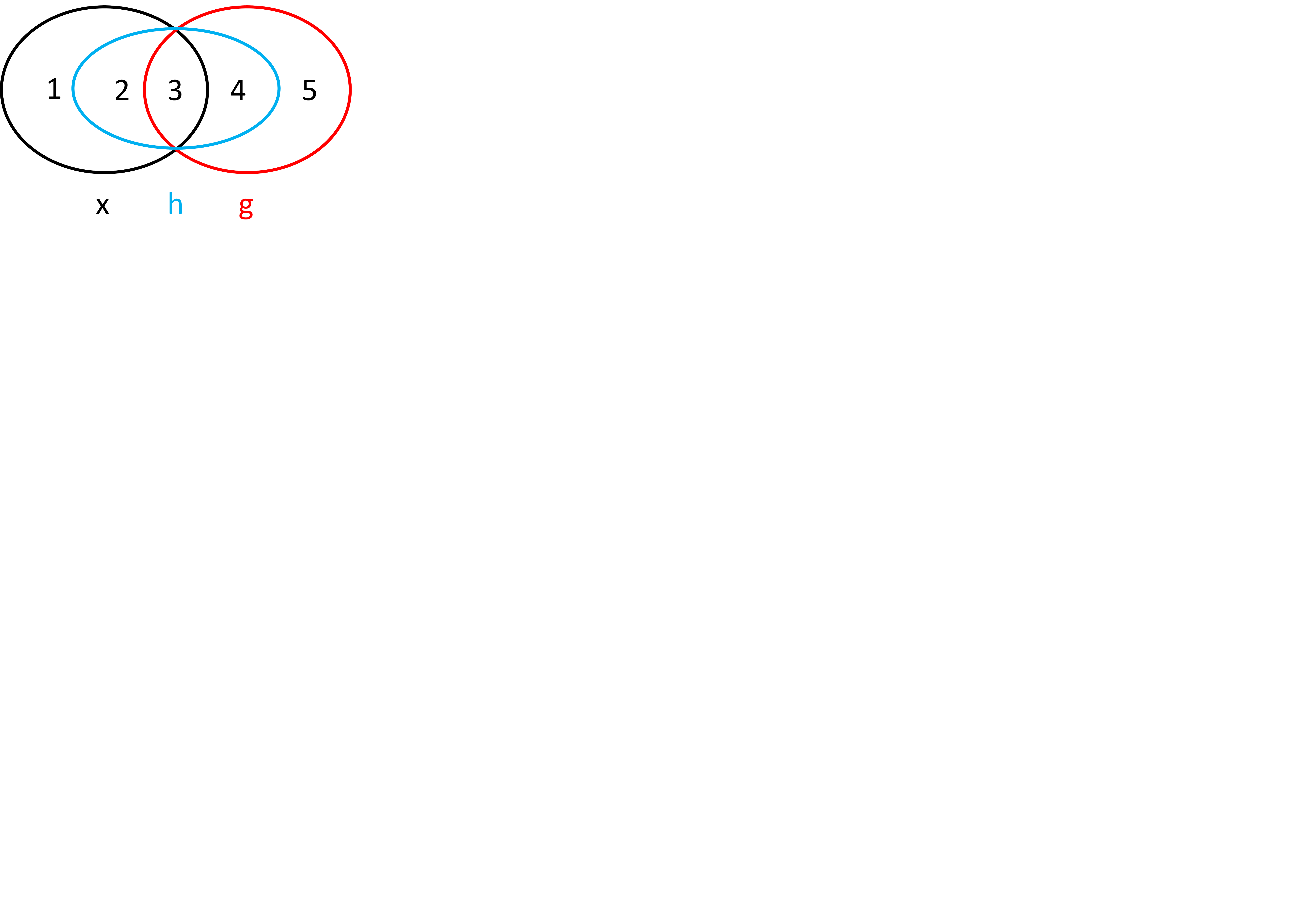}
    \vspace{-.2cm}
    \caption{Venn diagram of KI. $\rvx$ and $\rvh$ are random variables for entity representations of base LM and enhanced LM. $\rvg$ is the random variable for the local graph structure.}
    \label{fig:kidef}
    \vspace{-.2cm}
\end{wrapfigure}
Definition~\ref{def_ki} can be intuitively visualized by Figure~\ref{fig:kidef}.
Ideally, if most knowledge is successfully integrated without much forgetting of the old knowledge, regions $2$ and $4$ are large. We have $\MI(\rvh;\rvx) \!\approx\! \MI(\rvx; \rvx)$, and $\MI(\rvh;\rvg) \!\approx\! \MI(\rvg; \rvg)$. 
If little new knowledge has been integrated, i.e., catastrophic remembering happens, region $4$ is small, and we have $\MI(\rvh;\rvg) \!\approx\! \MI(\rvx; \rvg)$. Similarly, if catastrophic forgetting happens, region $2$ is small and we have $\MI(\rvh;\rvx) \!\approx\! \MI(\rvx; \rvg)$.

\subsection{Knowledge Integration Simulation}
Note that $f$ in Definition~\ref{def_ki} shows how KI happens and is an unknown ground-truth transformation that depends on many factors such as the base LM, KI corpus, and KI methods.
Thus, we propose an \textit{approximated transformation} $f'$ to fit $f$ such that it can simulate and approximate $f$ with high accuracy (\S\ref{section:method:transformation}). 
However, high accuracy cannot promise interpretation.
To interpret KI in a fine-grained way, we propose another \textit{interpretable transformation} $f''$ such that it can be as close to $f$ as $f'$ under the MI measurement  (\S\ref{section:method:simulation}).
\begin{figure*}[!htbp]
    \centering
    \includegraphics[scale=.60]{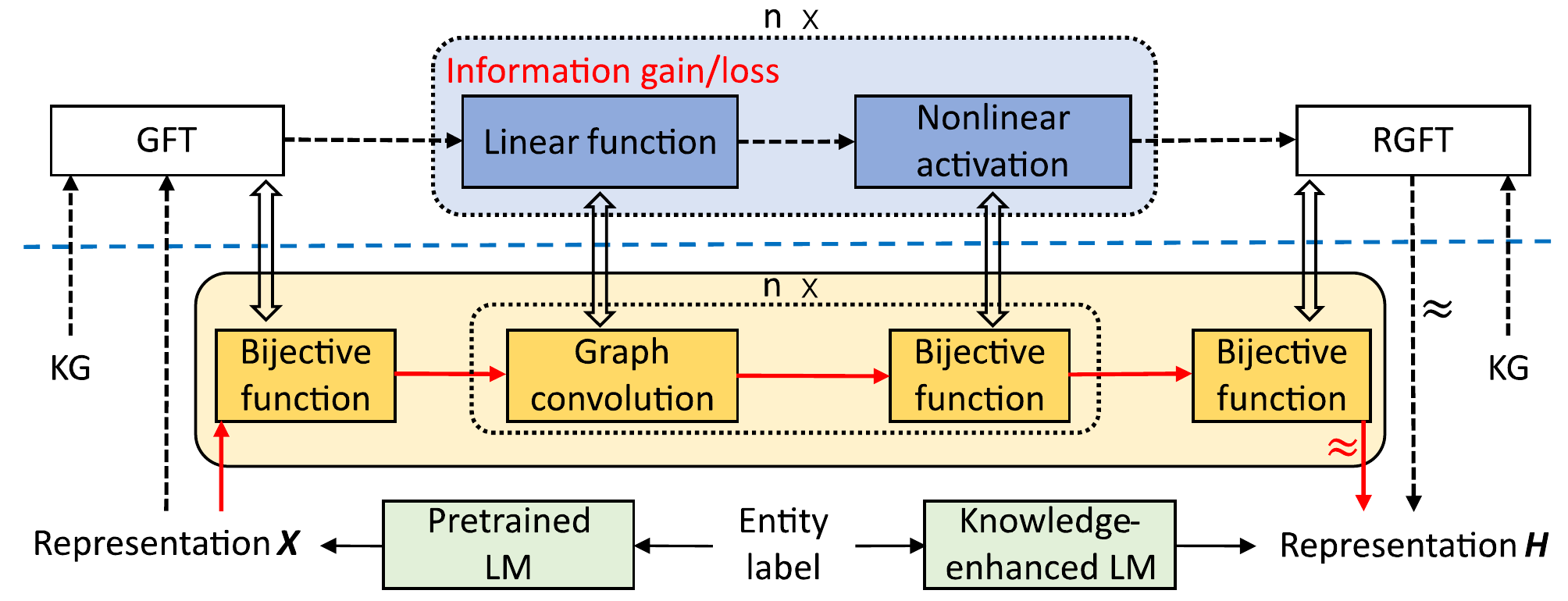}
    \vspace{-.1cm}
    \caption{Illustration of the KI simulation. The part above the horizontal blue dashed line represents graph spectral domain (i.e., KG space). GFT and RGFT are graph Fourier transformation and its inverse transformation.  Black dashed arrows show the \textit{approximated transformation} (i.e., $f'$), which can promise the approximation accuracy. Red arrows show the \textit{interpretable transformation} with graph convolutions (i.e., $f''$), which can promise the accuracy under the MI measurement and interpretability.}
    \vspace{-.2cm}
    \label{fig:intuition}
\end{figure*}
Figure~\ref{fig:intuition} briefly illustrates the idea.
The shown black dashed transformation (i.e., $f'$) can simulate KI with arbitrary accuracy. 
The solid red lines represent the interpretable transformation (i.e., $f''$) using graph convolutions, which could promise accuracy with the MI measurement and it is interpretable.
Below we introduce details of the two transformations.

\subsubsection{Approximated Transformation}\label{section:method:transformation}
Note that samples of $\rvg$ are non-Euclidean (local graph structures) while samples of $\rvx$ and $\rvh$ are vectors. To understand how $\rvx$ is transformed to $\rvh$ by integrating $\rvg$, we first map $\rvx$ and $\rvh$ to a new space related to $\rvg$.
Graph Fourier transform (GFT)~\citep{spectrum_smoothness_sandryhaila14} can be used to transform the entity representation $\rvx$ from the spatial domain to the graph spectral domain (KG space). 
We denote the transformation of $\rvx$ to the KG space as $\GFT(\rvx),$ and its inverse transformation as $\RGFT(\GFT(\rvx)) = \rvx$. Formal definition can be found in Appendix~\ref{appendix:gft_def}.
Using GFT, we will look at the change from $\rvx$ to $\rvh$ in the KG space, and construct an approximated transformation to simulate the KI process there.

\begin{theorem}[Approximation] \label{th_t1}
    Given a base LM and its enhanced LM, suppose that $\MI(\rvx;\rvg)<\MI(\rvh;\rvg)$. Then, for any $\epsilon > 0$, there exists an $n$-layer neural network ${\rm NN}^n(\cdot)$ such that 
    \begin{equation}
        |f(\rvx, \rvg) - \RGFT({\rm NN}^n(\GFT(\rvx)))|<\epsilon. \nonumber
    \end{equation}
\end{theorem}
The proof can be found in Appendix~\ref{appendix:theorem1}. Theorem~\ref{th_t1} shows that there exists an approximated transformation composed of GFT and a neural network that can simulate $f$ with arbitrary accuracy. 
However, the transformation $f'$ cannot provide specific insights about spatial samples. For example, it cannot show which set of knowledge triples contribute to KI.
Thus, to interpret KI, we develop a new transformation which can promise both accuracy and interpretability.

\subsubsection{Interpretable Transformation}\label{section:method:simulation}
In order to achieve an interpretable transformation with high accuracy, we make use of the invariance property of MI~\citep{invariance_mi_kraskov04}, i.e., the property that bijective functions would not change the MI. 
If we can change the metric in Theorem~\ref{th_t1} from $L1$ norm to MI, and replace operations in $f'(\rvx, \rvg)$ that change MI by other equivalent and interpretable operations, we can obtain an interpretable transformation to simulate KI.

Graph convolutions~\citep{sharedparam_defferrard16} are often used to model the relational information in KGs by filters (i.e., kernels), where entities aggregate information from their neighbors and pass the information based on the KG structure. 
Let ${\rm GC}(\cdot)$ denote the convolution operation on $\mathcal{G}$. Formal definition of ${\rm GC}$ can be found in Appendix~\ref{appendix:gft_def}.
If we run graph convolutions with attention~\citep{gat_velickovic18}, the information flow on graphs can be indicated by attention coefficients~\citep{zheng-etal-2020-document,smoothness_guoji20}, which can be used for interpretation.
Thus, we propose an interpretable transformation to simulate KI process using graph convolutions with attention mechanism.

\begin{theorem}[Interpretation] \label{th_t2}
    Given a base LM and its enhanced LM, let $\MI(\rvx;\rvg) < \MI(\rvh;\rvg)$. Denote that $f'(\rvx,\rvg)=\RGFT({\rm NN}^n(\GFT(\rvx)))$. 
    Let $\MLP_b(\cdot)$ denote the bijective MLP layer and ${\rm GC}(\cdot)$ denote the graph convolution on KG $\mathcal{G}$. There exists $f''(\rvx,\rvg)=\MLP_b({\rm GC}^n(\MLP_b(\rvx)))$, where ${\rm GC}^n$ is composed of $n$ repeat components as $\MLP_b({\rm GC}(\cdot))$, such that
    \begin{equation}
        \MI(f''(\rvx, \rvg); f(\rvx, \rvg)) = \MI(f'(\rvx, \rvg); f(\rvx, \rvg)). \nonumber
    \end{equation}
\end{theorem}
The proof can be found in Appendix~\ref{appendix:theorem2}. 
Note that the equality above becomes approximate equality when graph filters are approximated filters. For example, GCN~\citep{gcn_thomas17} as well as its variants (e.g., GAT~\citep{gat_velickovic18}) runs graph convolutions using localized first-order approximated filters. 
Theorem~\ref{th_t2} shows that we can use graph convolution operations to gain interpretability without the loss of MI. 
Assuming that MI can be used to measure knowledge in LMs as defined before, the interpretable transformation in Theorem~\ref{th_t2} can promise accuracy as well.

\subsection{Knowledge Integration Interpretation}\label{section:gcs:ga}
As shown in Theorem~\ref{th_t2}, $f''$ is composed of $n$ graph convolution operations and $n+2$ bijective functions (bijective MLPs). 
Since MI does not change with bijective functions, the information change (i.e., $\MI(\rvx;\rvg) \rightarrow \MI(\rvh;\rvg)$) can only happen in graph convolutions in $f''(\rvx, \rvg)$.
We then use graph attention coefficients in the graph convolutions to interpret how the information flows in a KG. Based on the information flow, we can interpret the integration of knowledge triples.

\begin{figure}[!htbp]
    \centering
    \includegraphics[scale=.14]{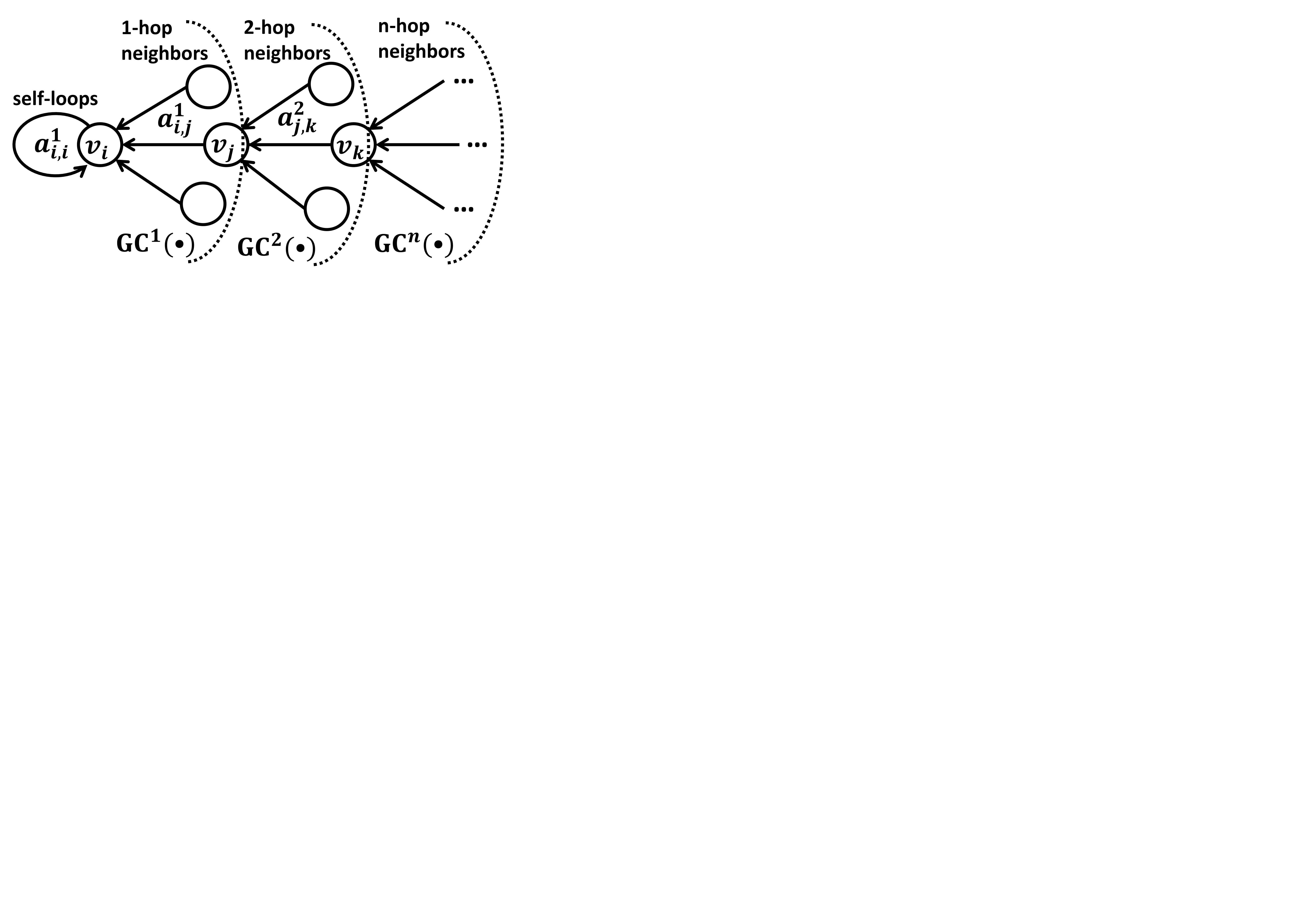}
    \vspace{-.2cm}
    \caption{Information flow on knowledge graph with respect to the entity $v_i$. Information from $n$-hop neighbors is aggregated with $n$-{th} graph convolution. Here, the aggregation number $n$ corresponds to the number of layers of the neural network ${\rm NN}^n(\cdot)$ in Theorem~\ref{th_t2}.}
    \vspace{-.2cm}
    \label{fig:gcs:infoflow}
\end{figure}
Figure~\ref{fig:gcs:infoflow} illustrates the information flow on KG with respect to $v_i$. Given a stacking of $n$ graph convolutions, the $i$-th graph convolution aggregates information from the $i$-hop neighbors. After the transformation, $v_i$ keeps $a_{i,i}^{1}$ $(0 \!<\! a_{i,i}^{1} \!<\! 1)$ of its original information, and aggregates $1\!-\!a_{i,i}^{1}$ of information from its $n$-hop neighbor entities.
For example, if $n\!=\!2$, the $2$-hop neighbor $v_k$ contributes $a_{i,j}^{1} \cdot a_{j,k}^{2}$ proportion of information to $v_i$.

We now use the information flow for interpretation, where attention coefficients on self-loops can be used to show if catastrophic remembering and catastrophic forgetting happened. For example, given entity $v_i$, if $a_{i,i}^{1} \!\approx\! 0$, it means entity $v_i$ did not keep any information of itself from the base LM in the enhanced LM, which means that catastrophic forgetting happens to $v_i$. Similarly, if $a_{i,i}^{1} \!\approx\! 1$, it means that catastrophic remembering happened to $v_i$.
Attention coefficients on the triples can be used to show if they are captured or not during KI. For example in Figure~\ref{fig:gcs:infoflow}, if $a_{i,j}^{1} \!\approx\! 1$, it means that from the base LM to the enhanced LM, entity $v_i$ became more closely associated with $v_j$. This indicates that the knowledge triple $(v_i, r, v_j)$ is newly integrated during KI. 
In our experiment, we regard knowledge triples with $a^{1}_{i,j} > 0.1$ as integrated triples, and others as non-integrated ones. Entities with $a^{1}_{i,i}<0.1$ imply that catastrophic forgetting happened, and correspondingly, $a^{1}_{i,i}>0.9$ implies that catastrophic remembering happened.\footnote{Users may choose different thresholds. We heuristically set these thresholds in our analysis for computational reasons.}

\section{Graph Convolution Simulator}\label{section_gcs}
So far, we have shown that there exists an interpretable transformation that can simulate KI theoretically. 
Next, we describe how to optimize the transformation (e.g., the MLP layers, and graph convolution layers) and introduce a practical implementation of our final GCS probe model. 

To implement the probe in practice, we make three approximations in GCS: two in the model design and one in optimization. We design experiments in the next section to make sure that GCS works well empirically.
The transformation as described in Theorem~\ref{th_t2} is implemented by bijective MLP layers and graph convolutional layers with attention. We show that if the weight matrix in MLP is a square matrix, the function is bijective. Formal description and proof are in Appendix~\ref{appendix:bijection}. For graph convolutions with attention, we use approximated graph attention filters similar to~\citep{agnn_thekumparampil18, gat_velickovic18} (Approximation 1). 
Then, we assume that the knowledge being integrated in the LM can be expressed as triples. 
In other words, we do not consider multi-hop knowledge triples (e.g., the $2$-hop knowledge triple $(v_i, r, v_k)$ in Figure~\ref{fig:gcs:infoflow}) in KI. Thus, we design GCS on the simple condition, where there is only one graph convolution operation (Approximation 2). Therefore, GCS is designed as two bijective MLP layers and one graph attention layer in between as:
\begin{equation}
    \GCS_{\theta_1}(\cdot) = \MLP({\rm GC}(\MLP(\cdot))), \nonumber
\end{equation}
where $\MLP(\cdot)$ is the bijective MLP layer and ${\rm GC}(\cdot)$ is the graph convolutional layer with attention on $\mathcal{G}$. Given an entity $v_i$ and its neighbors $\mathcal{N}_{v_i}$, we can write the graph convolutional layer as\footnote{For notation simplicity, we define $a_{i,j} \coloneqq a^{1}_{i,j}$.}:
\begin{equation}\small
\begin{aligned}
	{\rm GC}(\vx_i) & = \sigma \left( \sum_{v_j \in \mathcal{N}_{v_i}\cup \{v_i\} }  a_{i,j} \mW^{V} \vx_j \right), \nonumber \\
	a_{i,j} & = \text{softmax}\left( \frac{ ({\mW^{Q} \vx_i}) \cdot ({\mW^{K} \vx_j}) }{  t } \right). \nonumber
\end{aligned}
\end{equation}
Here, $\vx_i$ is the entity representation of $v_i$ from the base LM. The activation function $\sigma(\cdot)$ is ${\rm ELU}(\cdot)$, and $\mW^{V}$ is a weight matrix. $a_{i,j}$ is the attention coefficient on the relation that connects $v_i$ and $v_j$. $\mW^{Q}$ and $\mW^{K}$ are two parameter matrices in the graph attention. $\text{softmax}(\cdot)$ is the edge-wise softmax function with respect to node $v_i$. The temperature $t$ is a hyperparameter that controls the attention distribution to be hard or soft. 

We optimize GCS using the MI between its outputs and the entity representations from the enhanced LM:
\vspace{-2mm}
\begin{equation}\label{eq:gcsobj}
    \mathcal{L}=-\MI({\GCS_{\theta_1}}(\rvx);\rvh).
\end{equation}
In practice, we maximize the compression lemma lower bound of MI instead as introduced in~\citet{mine_belghazi18}.
More details can be found in Appendix~\ref{appendix:gcs_param}.
In practice, there may be a gap between the ground-truth MI and the compression lemma lower bound, and a stochastic optimization (e.g., Adam~\citep{adamoptimizer_kingma15}) may not converge to the optimal point. Thus, GCS may not fit $f''(\rvx, \rvg)$ perfectly (Approximation 3).

\section{Experiments}\label{section:exp}
We begin by reviewing ERNIE ~\citep{zhang-etal-2019-ernie} and K-Adapter ~\citep{wang-etal-2021-k}. Knowledge is integrated in \textit{entity-wise} manner in ERNIE, and \textit{triple-wise} manner in K-Adapter. 
%

\begin{figure*}[!htbp]
	\centering
	\subfigure{\includegraphics[scale=.15]{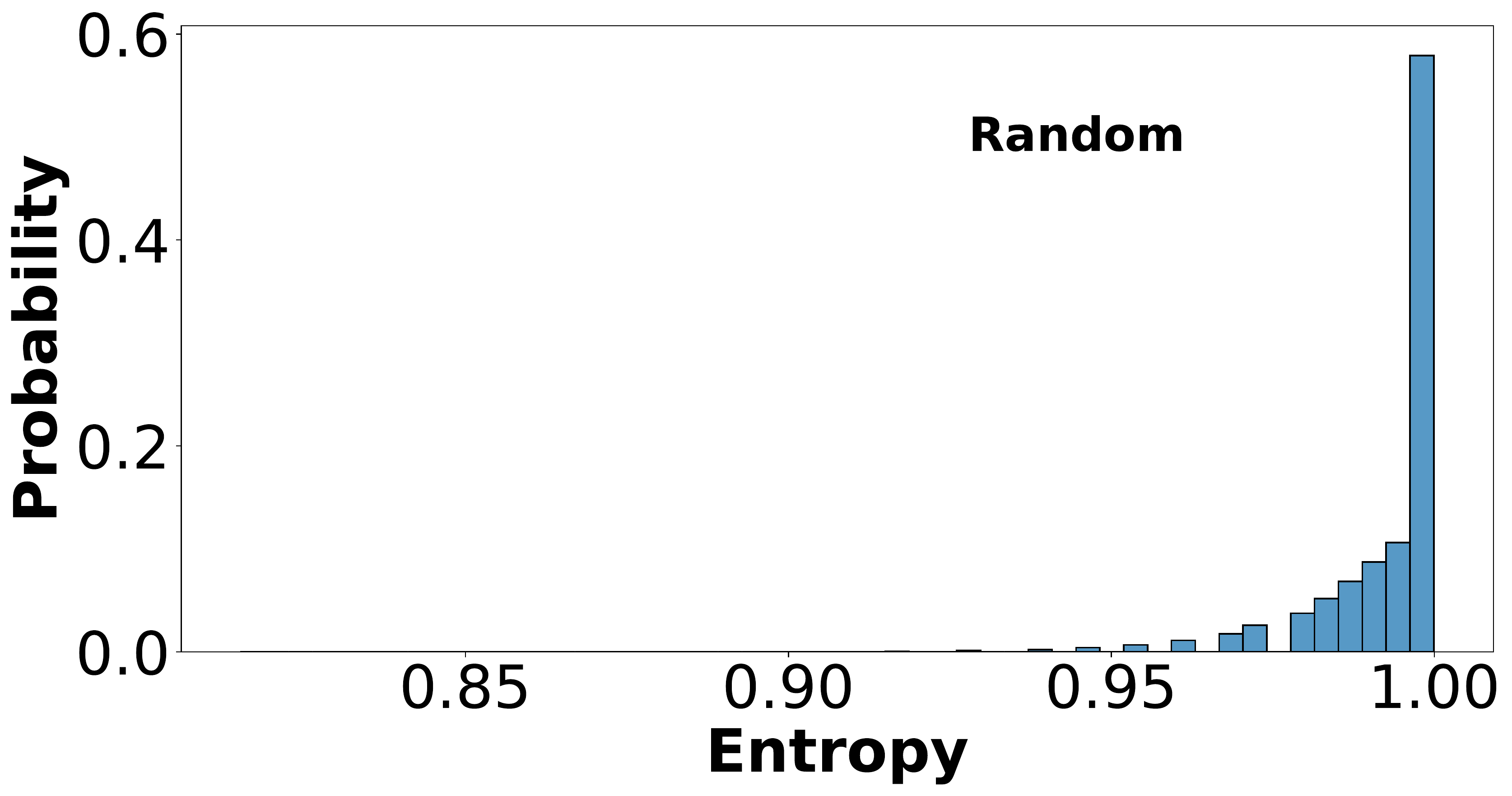}}
	\vspace{-.2cm}
	\subfigure{\includegraphics[scale=.15]{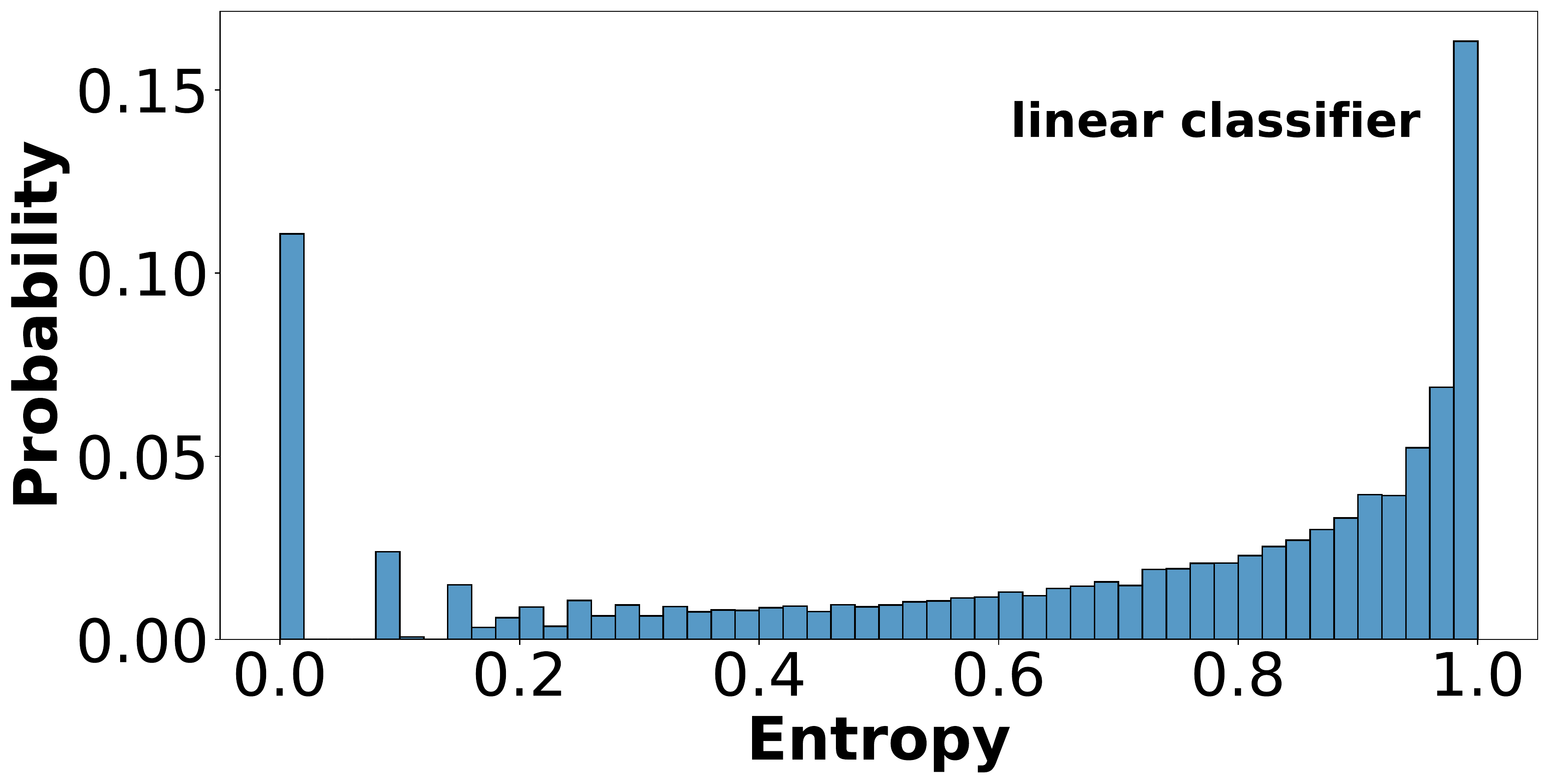}}
	\subfigure{\includegraphics[scale=.15]{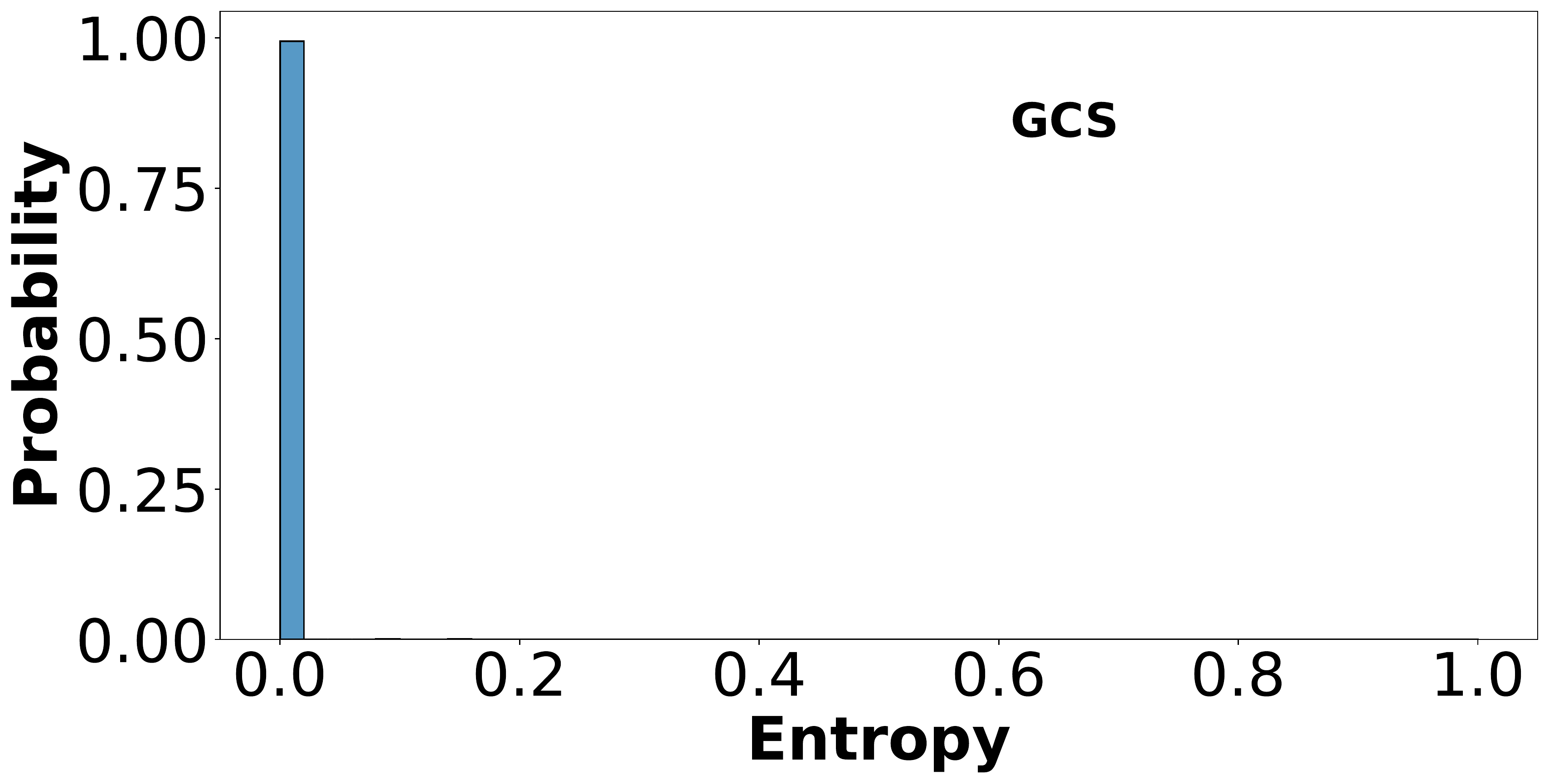}}
	\caption{The histogram of entropy for all entities (K-Adapter), where $x$-axis shows the entropy value and $y$-axis shows the empirical probability (i.e., frequency) of entities. The random guess strategy is also included for comparison. We can find that the KI interpretation results using linear probes have fairly large variance, which are similar to the results of random guessing. But our GCS can provide stable interpretations.}
	\vspace{-.2cm}
	\label{fig:probe_var}
\end{figure*}

\paragraph{ERNIE.}
ERNIE integrates knowledge into BERT~\citep{devlin-etal-2019-bert} using a Wikipedia corpus and Wikidata KG. 
As no alignment is provided between sentences in the Wikipedia corpus and entities in the Wikipedia KG, ERNIE uses TAGME~\citep{ernie_tagme_ferragina10} to extract entity mentions in sentences and aligns them with corresponding entities in KGs. A new objective is designed for KI in addition to the standard MLM and NSP objectives: alignments in the input text are randomly masked, and the model is asked to select aligned entities from KGs. When ERNIE finds the aligned entity, its KG embedding obtained from~\citet{transe_kge_bordes13} is integrated into the hidden representations of BERT. 

\paragraph{K-Adapter.}
K-Adapter takes RoBERTa~\citep{lm_roberta_liu} as the base LM and inserts three new layers into RoBERTa to integrate factual knowledge. The final output is concatenated with the output of RoBERTa.\footnote{Note that here we only consider factual knowledge, thus, the linguistic Adapter is not used.} During the integration, parameters of RoBERTa are frozen, only parameters in the adapter are updated. K-Adapter uses the T-REx-rc~\citep{elsahar-etal-2018-rex} dataset for KI, which has an alignment of sentences with knowledge triples in Wikidata.
For the KI objective, K-Adapter decides whether certain relations exist or not, and classifies relation labels given the aligned sentence. 

\subsection{GCS Verification}\label{exp:verification}
Next, we design a set of experiments to show that GCS can provide reasonable interpretations for KI. We first compare GCS with the linear probe with respect to variance of the interpretation results (\S\ref{exp:verify:variance}). We show that the linear probe does not work in interpreting KI, but GCS can provide stable interpretations. 
Then, we verify GCS based on the KG used in KI (\S\ref{exp:verify:kd_dataset}). We drop knowledge triples that are identified as \textit{non-integrated} by GCS during KI and show that it would not affect the performance of enhanced LMs. 
Third, we verify GCS based on the downstream tasks (\S\ref{exp:verify:downstream_task_performance}). We show that enhanced LMs perform well on the test samples that contain \textit{integrated} knowledge triples that are identified by GCS, and vice versa.

\subsubsection{Variance of interpretation results}\label{exp:verify:variance}
As alluded to earlier, linear probes do not work for large-scale factual knowledge interpretation. We support this claim by evaluating the variance of interpretation results. We use the entropy of probe results to test if the interpretation is stable. 

\paragraph{Setting.} We run a linear probe and GCS to detect if a knowledge triple is \textit{integrated} or \textit{non-integrated}  $100$ times with different random seeds.\footnote{We use $2$-fold cross validation for interpretation, where knowledge triples are randomly split into two even parts for training and testing.} We calculate the entropy of these results to estimate the probe variance. 
If a triple is classified as ``unlearned'' in the base LM but ``learned'' in the enhanced LM, we regard it as \textit{integrated} and if it is classified as ``unlearned'' in both base LM and enhanced LM, we say the triple is \textit{non-integrated}.
For GCS, as introduced in \S\ref{section:gcs:ga}, if the attention coefficient for the triple is larger than $0.1$, it is deemed as \textit{integrated}, else not.

\paragraph{Results.} A histogram of the entropy of all the entities is shown in Figure~\ref{fig:probe_var}. 
For a random guess strategy, the entropy of most knowledge triples is around $1$, which means the results are highly unstable. For the linear probe, we find that only $10\%$ of knowledge triples are interpreted in a stable manner. The other $90\%$ triples have large entropy (their interpretation is similar to random guessing).
On the other hand, for GCS, 
most knowledge triples have stable interpretation results. This shows that GCS can indeed provide more reliable interpretations.

\subsubsection{Verification via the KI corpus}\label{exp:verify:kd_dataset}
As our second verification experiment, we only use the factual knowledge that is identified as \textit{integrated}
by GCS to enhance BERT and RoBERTa to get ERNIE (drop-UE) and K-Adapter (drop-UE). 
Then we judge if the GCS interpretation was reasonable using two downstream tasks, where enhanced LMs outperform base LMs most significantly.
If GCS can interpret the KI process well, the performance of the drop-UE versions should be roughly the same as that of ERNIE/K-Adapter.

\paragraph{Setting.} This experiment comprises of three steps. First, we use GCS to interpret the KI process in ERNIE and K-Adapter, and identify triples or entities that are integrated successfully. Second, we re-enhance BERT/RoBERTa to get ERNIE (drop-UE) / K-Adapter (drop-UE) only using the entities/triples that are identified as \textit{integrated}. Third, we finetune ERNIE/K-Adapter and their drop-UE versions on two downstream tasks.

After we get our interpretation results, we only keep the KI corpus aligned with the \textit{integrated} knowledge.\footnote{We only keep $61.72\%$ entity embeddings (obtained by~\citet{transe_kge_bordes13}) for ERNIE (drop-UE), and $10.09\%$ KI corpus (i.e., natural sentences) for K-Adapter (drop-UE). Detailed statistics are in Table~\ref{tab:drop} in Appendix~\ref{appendix:additional_statistics}.} We finetune the models on two downstream tasks about entity typing: OpenEntity~\citep{choi-etal-2018-ultra} and FIGER~\citep{ling-etal-2015-design}. Implementation details of GCS, ERNIE, and K-Adapter can be found in Appendix~\ref{appendix:gcs_param} and Appendix~\ref{appendix:lm_param}.

\begin{table}[!htbp]
	\caption{Performance of ERNIE, K-Adapter and their drop-UE versions on the entity typing downstream tasks. We can find that dropping a large amount of non-integrated knowledge would not affect enhanced LMs' performance much on knowledge-intensive downstream tasks.}
	\vspace{-.2cm}
	\smallskip
	\label{tab:verify_input}
	\centering
	\resizebox{1.\columnwidth}{!}{
		\smallskip\begin{tabular}{c|ccc|ccc}
			\toprule
			\multirow{2}*{Model} & \multicolumn{3}{c}{OpenEntity} & \multicolumn{3}{c}{FIGER} \\
			\cline{2-7}
			~ & P & R & F1-Micro & P & R & F1-Micro  \\
			\midrule
            ERNIE & \textbf{78.24} & 68.75 & 73.19 & \textbf{77.39} & \textbf{65.81} & \textbf{71.13} \\
            ERNIE (drop-UE) & 78.11  & \textbf{71.43}  & \textbf{74.62} {\color{red} $\uparrow$} & 77.38  & 64.90 & 70.60 {\color{green} $\downarrow$} \\
            \hline
            K-Adapter & \textbf{76.63} & 75.26 & 75.94 & \textbf{67.50} & 88.79 & \textbf{76.69} \\
            K-Adapter (drop-UE) & 75.95  & \textbf{75.95} & \textbf{75.95} {\color{red} $\uparrow$} & 67.29  & \textbf{88.88}  & 76.59 {\color{green} $\downarrow$} \\
			\bottomrule
			\hline
		\end{tabular}
	}	
	\vspace{-.2cm}
\end{table} 
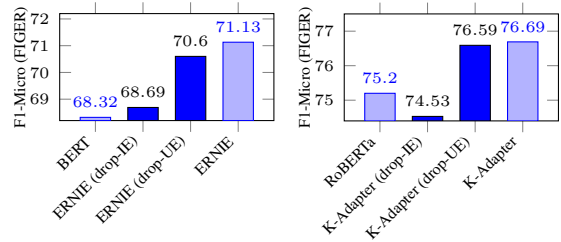
\begin{figure}[!htbp]
	\centering
	\pgfplotsset{width=4.4cm, height=3.cm}
	\vspace{-.2cm}
	\subfigure{
		\begin{tikzpicture}[font=\tiny]
		\begin{axis}[ybar, bar shift=0pt, bar width=0.4cm, enlarge x limits=0.25,
		            x tick label style={rotate=45, anchor=east, align=left},
                    y label style={at={(axis description cs:-0.1,0.5)},anchor=south},
                    ylabel={F1-Micro (FIGER)}, 
                    symbolic x coords={BERT, \quad\,\,\, ERNIE (drop-IE), \quad ERNIE (drop-UE), ERNIE}, 
                    xtick={BERT, \quad\,\,\, ERNIE (drop-IE), \quad ERNIE (drop-UE), ERNIE},
                    nodes near coords, 
                    nodes near coords align={vertical},
                    ymin=68.2, ymax=72.2]
		\addplot coordinates {(BERT, 68.32) (ERNIE, 71.13)};
		\addplot[fill=blue] coordinates {(\quad\,\,\, ERNIE (drop-IE), 68.69) (\quad ERNIE (drop-UE), 70.60)};
		\end{axis}
		\end{tikzpicture}
	}%
	\subfigure{
		\begin{tikzpicture}[font=\tiny]
		\begin{axis}[ybar, bar shift=0pt, bar width=0.4cm, enlarge x limits=0.25,
		            x tick label style={rotate=45, anchor=east, align=left},
                    y label style={at={(axis description cs:-0.1,0.5)},anchor=south},
                    ylabel={F1-Micro (FIGER)}, 
                    symbolic x coords={RoBERTa, K-Adapter (drop-IE), K-Adapter (drop-UE), K-Adapter}, 
                    xtick={RoBERTa, K-Adapter (drop-IE), K-Adapter (drop-UE), K-Adapter},
                    nodes near coords, 
                    nodes near coords align={vertical},
                    ymin=74.4, ymax=77.5]
		\addplot coordinates {(RoBERTa, 75.20) (K-Adapter, 76.69)};
		\addplot[fill=blue] coordinates {(K-Adapter (drop-IE), 74.53) (K-Adapter (drop-UE), 76.59)};
		\end{axis}
		\end{tikzpicture}
	}
	\vspace{-.6cm}
	\caption{Performance of BERT, RoBERTa, ERNIE, K-Adapter, and their dropped versions on the FIGER dataset. We can find that even if there are negative effects on the performance, they are marginal that can be ignored.}
	\vspace{-.4cm}
	\label{fig:verify_input}
\end{figure}

\paragraph{Results.}
From Table~\ref{tab:verify_input}, we find that even if we drop a large amount of KI data in this way, the performance of drop-UE versions on entity typing task is roughly the same as original versions. For the OpenEntity dataset, better performance is achieved. For the FIGER dataset, the performance of drop-UE versions is slightly worse. We also report the performance of BERT, RoBERTa, and two drop-IE\footnote{The drop-IE versions mean that we drop the subset aligned with integrated factual knowledge identified by GCS and use the rest to enhance BERT and RoBERTa .} versions in Figure~\ref{fig:verify_input}. We find that compared to dropping KI corpus aligned with \textit{integrated} knowledge, dropping KI corpus aligned with \textit{non-integrated} knowledge achieves much better performance. Thus, we verify that GCS provides reasonable interpretations.

\subsubsection{Verification via the downstream task} \label{exp:verify:downstream_task_performance}
We use the performance of ERNIE and K-Adapter on 
downstream tasks to verify GCS. If GCS can reasonably interpret KI, enhanced LMs should perform better on the test set with samples aligned with the \textit{integrated knowledge}, and vice versa. 

\paragraph{Setting.} We first align entities in the KI corpus and the OpenEntity dataset based on their \textit{Wikidata Q identifier}.\footnote{\href{https://www.wikidata.org/wiki/Q43649390}{https://www.wikidata.org/wiki/Q43649390}} For the entity typing task (OpenEntity dataset), we drop samples in the finetuning test set that aligns with the \textit{integrated} knowledge and \textit{non-integrated} entities (called \textit{w/o-IE} test set and \textit{w/o-UE} test set), and test ERNIE and K-Adapter on the two dropped test sets. Detailed statistics can be found in the Table~\ref{tab:verify_output} in Appendix~\ref{appendix:additional_statistics}.

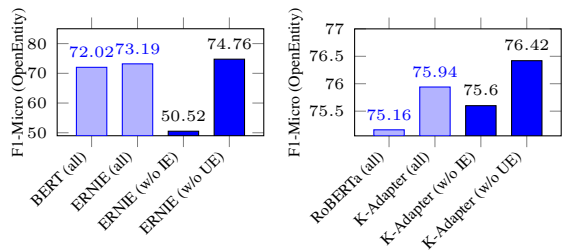
\begin{figure}[!htbp]
	\centering
	\pgfplotsset{width=4.3cm, height=3.cm}
	\vspace{-.4cm}
	\subfigure{
		\begin{tikzpicture}[font=\tiny]
		\begin{axis}[ybar, bar shift=0pt, bar width=0.4cm, enlarge x limits=0.25,
		            x tick label style={rotate=45, anchor=east, align=left},
                    y label style={at={(axis description cs:-0.1,0.5)},anchor=south},
                    ylabel={F1-Micro (OpenEntity)}, 
                    symbolic x coords={BERT (all), ERNIE (all), \quad\,\,\, ERNIE (w/o IE), \quad ERNIE (w/o UE)}, 
                    nodes near coords, 
                    nodes near coords align={vertical},
                    ymin=49, ymax=85, xtick={BERT (all), ERNIE (all), \quad\,\,\, ERNIE (w/o IE), \quad ERNIE (w/o UE)}]
		\addplot coordinates {(BERT (all), 72.02) (ERNIE (all), 73.19)};
		\addplot[fill=blue] coordinates {(\quad\,\,\, ERNIE (w/o IE), 50.52) (\quad ERNIE (w/o UE), 74.76)};
		\end{axis}
		\end{tikzpicture}
	}%
	\subfigure{
		\begin{tikzpicture}[font=\tiny]
		\begin{axis}[ybar, bar shift=0pt, bar width=0.4cm, enlarge x limits=0.25,
		            x tick label style={rotate=45, anchor=east, align=left},
                    symbolic x coords={RoBERTa (all), K-Adapter (all), K-Adapter (w/o IE), K-Adapter (w/o UE)}, 
                    y label style={at={(axis description cs:-0.2,0.5)},anchor=south},
                    ylabel={F1-Micro (OpenEntity)}, 
                    nodes near coords, 
                    nodes near coords align={vertical},
                    ymin=75.05, ymax=77, xtick={RoBERTa (all), K-Adapter (all), K-Adapter (w/o IE), K-Adapter (w/o UE)}]
		\addplot  coordinates {(RoBERTa (all), 75.16) (K-Adapter (all), 75.94)};
		\addplot [fill=blue] coordinates {(K-Adapter (w/o IE), 75.60) (K-Adapter (w/o UE), 76.42)};
		\end{axis}
		\end{tikzpicture}
	}
	\vspace{-.6cm}
	\caption{Performance of BERT, RoBERTa, ERNIE, K-Adapter on the OpenEntity dataset for different test sets (original versions and dropped versions). We can find that enhanced LMs perform better on test samples that contain successfully integrated knowledge.}
	\vspace{-.2cm}
	\label{fig:verify_output}
\end{figure}

\paragraph{Results.} As shown in Figure~\ref{fig:verify_output}, we find that for ERNIE, the difference is significant. The performance on the test set (w/o-IE) is more than $20$ F1 points worse than that on the complete test set (all). 
For K-Adapter, there is a drop in F1 on the w/o-IE set and increase in F1 on the w/o-UE set (albeit small). 
We hypothesize that this may be because of the differences in the finetuning objective and the KI objective\footnote{K-Adapter integrates knowledge in a triple-wise manner only using a small amount of adapter parameters.}, and because the knowledge integrated in K-Adapter may change during finetuning. 
These results show that GCS can reasonably interpret which set of knowledge is integrated.

\subsection{GCS Findings} \label{exp:findings}
After verifying GCS with three set of experiments, we analyze the interpretation results. We find that both ERNIE and K-Adapter integrate only few knowledge triples $(\approx 20-30\%)$. Detailed results can be found in Figure~\ref{fig:all_dist} in Appendix~\ref{appendix:additional_statistics}. 
Next, we classify the factual knowledge based on its relation types (in terms of their topology type and Wiki data type) and analyze how ERNIE and K-Adapter integrate knowledge with certain relation types.

\subsubsection{Analysis via relation topology} 
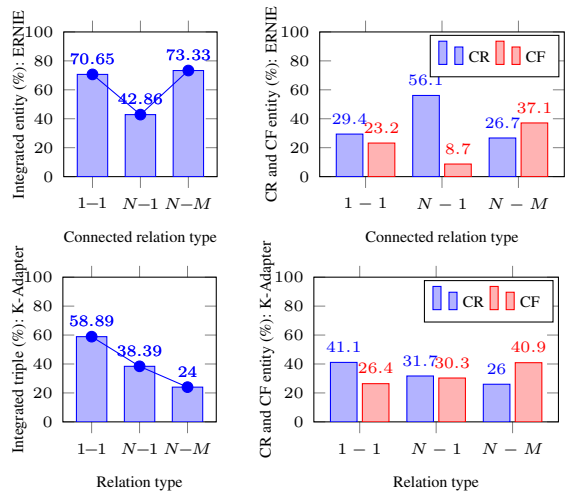
\begin{figure}[htbp]
	\centering
	\vspace{-.2cm}
	\subfigure{
	    \pgfplotsset{width=3.6cm, height=3.5cm}
		\begin{tikzpicture}[font=\tiny]
		\begin{axis}[ybar, bar shift=0pt, bar width=0.4cm, enlarge x limits=0.3,
		            xtick=data,
                    symbolic x coords={$1\!\!-\!\!1$, $N\!\!-\!\!1$, $N\!\!-\!\!M$}, 
                    y label style={at={(axis description cs:-0.17,0.5)},anchor=south},
                    xlabel={Connected relation type},
                    ylabel={Integrated entity (\%): ERNIE}, 
                    nodes near coords, 
                    nodes near coords align={vertical},
                    ymin=0., ymax=100]
		\addplot coordinates {($1\!\!-\!\!1$, 70.65) ($N\!\!-\!\!1$, 42.86) ($N\!\!-\!\!M$, 73.33)};
		\addplot [color=blue, mark=*, sharp plot] coordinates {($1\!\!-\!\!1$, 70.65) ($N\!\!-\!\!1$, 42.86) ($N\!\!-\!\!M$, 73.33)};
		\end{axis}
		\end{tikzpicture}
	}
	\vspace{-.7cm}
	\subfigure{
	    \pgfplotsset{width=4.8cm, height=3.5cm}
		\begin{tikzpicture}[font=\tiny]
		\begin{axis}[ybar, bar width=0.35cm, enlarge x limits=0.3, legend columns=-1,
		            xtick=data,
                    symbolic x coords={$1-1$, $N-1$, $N-M$}, 
                    y label style={at={(axis description cs:-0.12,0.5)},anchor=south},
                    xlabel={Connected relation type},
                    ylabel={CR and CF entity (\%): ERNIE}, 
                    nodes near coords, 
                    nodes near coords align={vertical},
                    ymin=0., ymax=100]
		\addplot coordinates {($1-1$, 29.4) ($N-1$, 56.1) ($N-M$, 26.7)};
		\addlegendentry{CR}
		\addplot coordinates {($1-1$, 23.2) ($N-1$, 8.7) ($N-M$, 37.1)};
		\addlegendentry{CF}
		\end{axis}
		\end{tikzpicture}
	}
	\subfigure{
	    \pgfplotsset{width=3.6cm, height=3.5cm}
		\begin{tikzpicture}[font=\tiny]
		\begin{axis}[ybar, bar shift=0pt, bar width=0.4cm, enlarge x limits=0.3,
                    symbolic x coords={$1\!\!-\!\!1$, $N\!\!-\!\!1$, $N\!\!-\!\!M$}, 
                    y label style={at={(axis description cs:-0.17,0.5)},anchor=south},
                    xlabel={Relation type},
                    ylabel={Integrated triple (\%): K-Adapter}, 
                    nodes near coords, 
                    nodes near coords align={vertical},
                    ymin=0., ymax=100, xtick={$1\!\!-\!\!1$, $N\!\!-\!\!1$, $N\!\!-\!\!M$}]
		\addplot coordinates {($1\!\!-\!\!1$, 58.89) ($N\!\!-\!\!1$, 38.39) ($N\!\!-\!\!M$, 24.00)};
		\addplot [color=blue, mark=*, sharp plot] coordinates {($1\!\!-\!\!1$, 58.89) ($N\!\!-\!\!1$, 38.39) ($N\!\!-\!\!M$, 24.00)};
		\end{axis}
		\end{tikzpicture}
	}
	\subfigure{
	    \pgfplotsset{width=4.8cm, height=3.5cm}
		\begin{tikzpicture}[font=\tiny]
		\begin{axis}[ybar, bar width=0.35cm, enlarge x limits=0.3, legend columns=-1,
                    symbolic x coords={$1-1$, $N-1$, $N-M$}, 
                    y label style={at={(axis description cs:-0.12,0.5)},anchor=south},
                    xlabel={Relation type},
                    ylabel={CR and CF entity (\%): K-Adapter}, 
                    nodes near coords, 
                    nodes near coords align={vertical},
                    ymin=0., ymax=100, xtick={$1-1$, $N-1$, $N-M$}]
		\addplot coordinates {($1-1$, 41.1) ($N-1$, 31.7) ($N-M$, 26.0)};
		\addlegendentry{CR}
		\addplot coordinates {($1-1$, 26.4) ($N-1$, 30.3) ($N-M$, 40.9)};
		\addlegendentry{CF}
		\end{axis}
		\end{tikzpicture}
	}
	\vspace{-.5cm}
	\caption{Analysis of KI interpretation results in terms of different relation topology. We can find that the degree of knowledge (types) integration is different for enhanced LMs using different KI methods.}
	\vspace{-.2cm}
	\label{fig:specific}
\end{figure}

We classify relations into three types based on their topology features. 
Following previous work~\citep{transe_kge_bordes13,tang-etal-2020-orthogonal}, we denote relations that connect two leaf nodes (entities) in the KG as $1\!-\!1$ relations, and relations that connect two center nodes (entities) in the KG as $N\!-\!M$ relations. Others are denoted as $N\!-\!1$ relations. An example can be found in Figure~\ref{fig:kg_relation} in Appendix~\ref{appendix:additional_statistics}.
We perform an analysis in terms of different types of relations and report the percentage of successfully integrated entities and triples for ERNIE and K-Adapter. For each relation type, we also present the percentage of connected entities that are catastrophically remembered or forgotten (CR and CF).

Figure~\ref{fig:specific} presents the results, and detailed statistics can be found in Table~\ref{tab:specific} in Appendix~\ref{appendix:additional_statistics}.
We find that for ERNIE, entities connected with complex relations (i.e., $N\!-\!M$ relations) are captured well.
However, K-Adapter shows different behaviors. It captures triples with simple relations (i.e., $1\!-\!1$ and $N\!-\!1$ relations) well. 
Note that ERNIE uses graph embeddings provided by the specially designed model~\citep{transe_kge_bordes13} for KI, while K-Adapter integrates knowledge into dedicated feedforward networks called adapters. 
This implies that structures are not well encoded into classic adapter modules, and we may need a better approach to integrate knowledge with complex structures into neural modules.
Besides, we find that for both ERNIE and K-Adapter, CR happens more often to entities in simple structures (i.e., $N\!-\!1$ relations), while CF is more common for entities in complex structures (i.e., connected to $N-M$ relations).

\subsubsection{Analysis via relation's Wiki features}
\begin{table}[!htbp]
	\caption{Analysis of KI interpretations via different relation labels for the T-REx-rc dataset (KI corpus used by K-Adapter). We can find that temporal knowledge is hard to get integrated.}
	\vspace{-.2cm}
	\smallskip
	\label{tab:label}
	\centering
	\resizebox{1.\columnwidth}{!}{
		\smallskip\begin{tabular}{c|ccc}
			\toprule
			\multirow{2}{*}{Relation label} & \multicolumn{3}{c}{T-REx-rc} \\
			\cline{2-4}
			~  & Wiki Count & Wiki data type & Integrated triple \\
			\midrule
            place of birth (LF) & 2,850,424 & Wikibase item & 10.95\% \\
            part of (LF) & 4,164,470 & Wikibase item & 17.25\% \\
            \hline
            date of death (TR) & 2,637,358 & Time & $<$0.01\% \\
            date of birth (TR) & 5,294,649 & Time & $<$0.01\%  \\
            \hline
            located in the administrative & \multirow{2}{*}{10,776,120} & \multirow{2}{*}{Wikibase item} & \multirow{2}{*}{6.13\%} \\
             territorial (HF) & ~ & ~ & ~ \\
            country (HF) & 14,174,811 & Wikibase item & 0.12\% \\
            \hline
            Total & - & - & 10.09\% \\
			\bottomrule
			\hline
		\end{tabular}
	}
	\vspace{-.2cm}
\end{table}
For further analysis, we select six relations aligned with roughly the same number of sentences in the T-REx-rc dataset\footnote{Statistics are in Table~\ref{tab:label_num} in Appendix~\ref{appendix:additional_statistics}} and categorize them into three groups based on the \textit{Wiki Count} and \textit{Wiki data type}\footnote{See the \href{https://www.wikidata.org/wiki/Wikidata:Database\_reports/List\_of\_properties/all}{Wikipedia page} for more details.}: low-frequency (LF) relations, time-related (TR) relations, and high-frequency (HF) relations.
From Table~\ref{tab:label}, we find that even if LF relations have roughly the same Wiki Count as TR relations, the temporal knowledge still cannot be integrated by K-Adapter. We speculate that this is because Transformer encoders do not capture information about time well~\citep{temporallm_dhingra21, zhou-etal-2021-temporal}. When comparing LF relations and HF relations, we find that if relations have small Wiki Count, knowledge triples are easily captured.

\begin{table}[!htbp]
	\caption{Examples of triples in the T-REx-rc dataset (KI corpus used by K-Adapter) with attention coefficients.}
	\vspace{-.2cm}
	\smallskip
	\label{tab:label_example}
	\centering
	\resizebox{1.\columnwidth}{!}{
		\smallskip\begin{tabular}{c|c}
			\toprule
			Knowledge triple & Attention coefficient  \\
			\midrule
            (Adam Smith, place of birth, Kirkcaldy) & $1.079 \times 10^{-1}$ \\
            (Lake Huron, part of, Great Lakes) & $1.742 \times 10^{-1}$\\
            \hline
            (Jean-Jacques Rousseau, date of death, 02 July 1778) & $1.729 \times 10^{-25}$ \\
            (Barack Obama, date of birth, 04 August 1961) & $6.827 \times 10^{-31}$\\
            \hline
            (Mauna Kea Observatory,  &  \multirow{2}{*}{$6.044  \times 10^{-2}$}  \\
            located in the administrative territorial, Hawaii) & ~\\
            (China, country, Mahalangur Himal) & $1.250  \times 10^{-3}$ \\
			\bottomrule
			\hline
		\end{tabular}
	}
	\vspace{-.2cm}
\end{table}
Randomly picked examples of knowledge triples are given in Table~\ref{tab:label_example}. We can find that for TR relations, they connect entities composed of numbers. The poor performance of language models in handling numbers~\citep{wallace-etal-2019-nlp} provides an alternative explanation for the observation that K-Adapter does not integrate triples with TR relations. For triples with LF and HF relations, we find that some entities connected to HF relations are very common (e.g, entity ``China'' is in the complex KG structure) compared to those connected to LF relations. These results are consistent with our findings in Figure~\ref{fig:specific} that knowledge of popular entities is not integrated.

\subsubsection{Can we further improve the KI quality?}
Finally, we attempt to answer the question: \textit{can we simply improve the quality of KI by increasing the amount of our aligned training corpora?} Intuitively, repeatedly learning a knowledge triple with several aligned sentences could increase the possibility of successful integration of this knowledge. 

\begin{figure}[!htbp]
    \centering
    \vspace{-.2cm}
    \includegraphics[scale=.24]{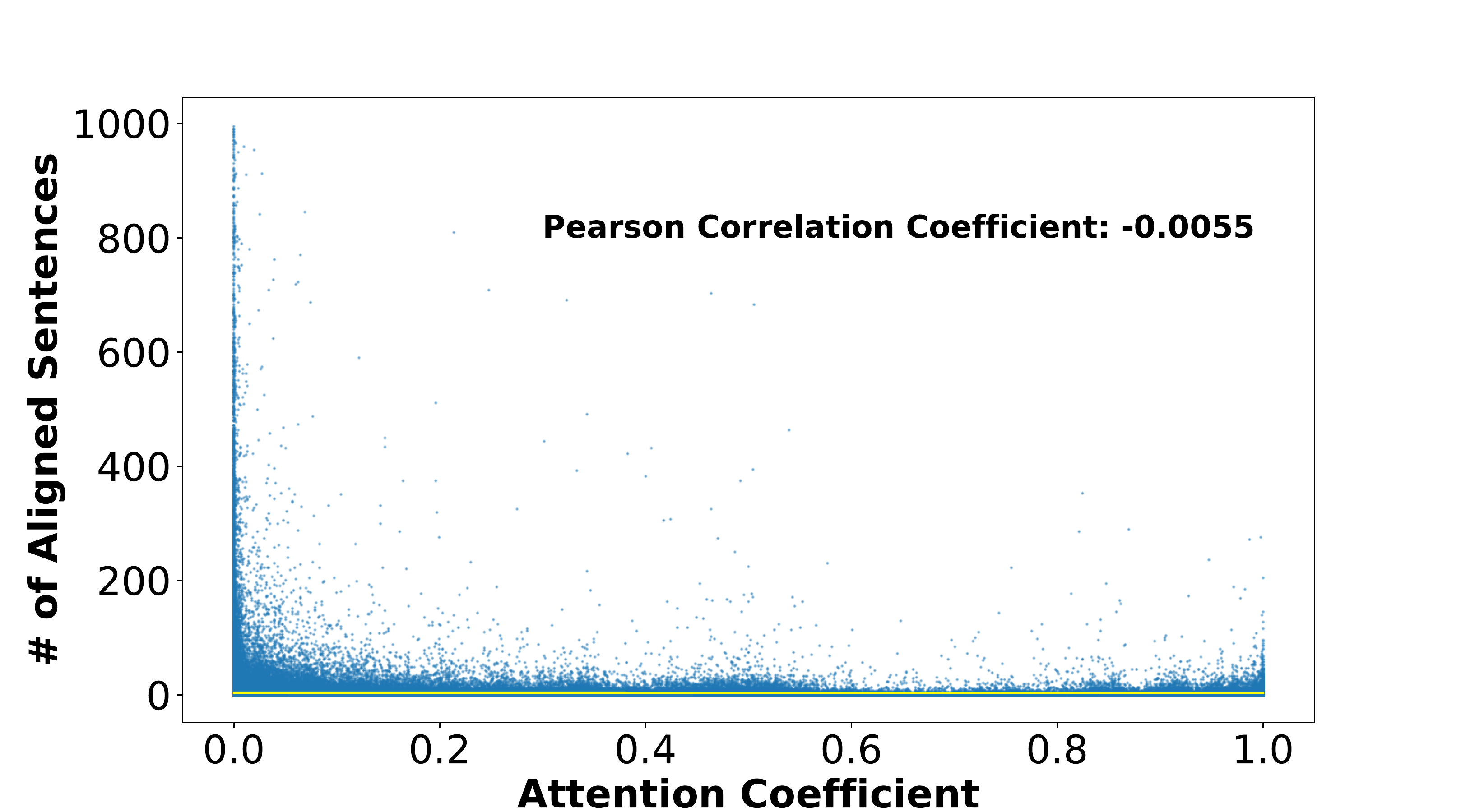}
    \caption{The correlation between the attention coefficient of the knowledge triple and its aligned sentence number in T-REx-rc dataset. We can find that there is no correlation between them, which means simply increasing the KI corpus could not be helpful for better KI quality.}
    \vspace{-.2cm}
    \label{fig:correlation}
\end{figure}

We answer this question by calculating correlation between the attention coefficients (i.e., success ratio of integration) for K-Adapter and number of aligned sentences (i.e., size of KI corpus) for knowledge triples in the T-REx-rc dataset.
Surprisingly, we find that this \textit{Pearson correlation} is $-0.0055$ (Figure~\ref{fig:correlation}). This shows that there is no apparent positive relationship between the KI quality and the size of the KI dataset. It suggests that simply increasing the size of the aligned dataset alone may not improve KI and we might need more fundamental advances to push the state-of-the-art in KI.

\section{Conclusion}
In this paper, through a series of theoretical results, we derived an information-theoretic probe that uses attention over knowledge graphs to interpret the knowledge integration process in LMs. In our experiments, we verified our probe model and used it to understand what knowledge has been integrated in two existing knowledge-enhanced language models, leading to some new findings about these models. 
We hope that our probe model would aid in better understanding and informed design of knowledge integration approaches for LMs. We have published the code and the demo to help users easily implement GCS for their own knowledge-enhanced LM interpretation.

\section*{Limitations}
There are some limitations of our work. 
We simplify the KG without considering relation information such as labels, since existing graph neural networks (e.g., R-GCN~\citep{rgcn_schlichtkrull18}) still cannot handle such large number of imbalanced distributed relations.
These can be considered by future works.
Besides, GCS only provides a way to interpret the knowledge integration. Once we have an understanding of the it, improving the integration quality still remains challenging.

\section*{Reproducibility Statement}
We have published the code, the demo, and the interpretation results. The design of GCS can be found in~\S\ref{section_gcs}. The implementation details about the knowledge integration for K-Adapter and ERNIE can be found in Appendix~\ref{appendix:lm_param}, and details about GCS can be found in Appendix~\ref{appendix:gcs_param}.

\section*{Ethics Statement}
While our probe models are not tuned for any specific real-world application, our methods could be used in sensitive contexts such as legal or health-care settings; and it is essential that any work that builds on our approaches undertake extensive quality-assurance and robustness testing before using it in their setting.

\section*{Acknowledgment}
We are grateful to the anonymous reviewers for their insightful comments and suggestions, which helped us significantly improve the paper. We also owe many thanks to Shehzaad Dhuliawala and Nico Daheim for their constructive advice on this paper. 
Yifan Hou is supported by the Swiss Data Science Center PhD Grant (P22-05). We also acknowledge support from an ETH Zurich Research grant (ETH-19 21-1) and a grant from the Swiss National Science Foundation (project \#201009) for this work.

\bibliographystyle{acl_natbib}
\bibliography{anthology,custom}

\clearpage
\appendix
\section{Notations}\label{appendix:notations}
The notations can be found in the below table.
\begin{table*}[!htbp]
	\caption{Notations and their descriptions}\smallskip
	\label{tab:notation}
	\centering
	\resizebox{1.6\columnwidth}{!}{
		\smallskip\begin{tabular}{c|c}
			\toprule
			Notation & Description \\
			\midrule
            $\mathcal{G}$ & The knowledge graph for KI \\
            $\mathcal{V}$ & The set of entities/nodes of KG \\
            $\mathcal{E}$ & The set of edges of KG \\
            $v_i$ & The entity/node indexed as $i$ in the KG \\
            $t_i$ & The entity label attached on $v_i$ \\
            ${\rm LM}(\cdot)$ & The language model, where the input is entity text, and the output is its representation \\
            $\mathcal{N}_{v_i}$ & The set of neighbors (entities/nodes) connected to $v_i$ \\
            $\mathcal{G}(v_i)$ & The local graph structure in terms of $v_i$ \\
            $\rvx$ & The random variable of the entity representation\\
            $\vx_i$ & The entity representations of $v_i$ \\
            $\rvg$ & The random variable of the local graph structure \\
            $\MI(\cdot;\cdot)$ & The mutual information between two random variables \\
            $\mA$ & The adjacency matrix of KG \\
            $|\mathcal{V}|$ & The number of entities/nodes in KG \\
            $\mathbb{R}$ & The set of real numbers \\
            $\mI$ & The identity matrix \\
            $\mD$ & The degree matrix of KG \\
            $\mL_n$ & The normalized Laplacian matrix \\
            $\text{diag}(\cdot)$ & The diagonalization operation \\
            $\mU$ & The matrix of eigenvectors \\
            $\mLambda$ & The diagonal matrix of eigenvalues \\
            $\lambda_i$ & The $i$-th eigenvalue \\
            $\mX$ & The set of entity representations in terms of $\mathcal{V}$ \\
            $C$ or $d$ & The dimension of entity representations; The number of channels \\
            $\GFT(\cdot)$ & The graph Fourier transformation \\
            $\RGFT(\cdot)$ & The inverse graph Fourier transformation \\
            $g_{\Theta}$ & The graph filter parameterized by parameter $\Theta$ \\
            $\mH$ & The entity representations given by a knowledge-enhanced LM \\
            $\rvh$ & The random variable of the entity representation given by a knowledge-enhanced LM \\
            $f(\cdot,\cdot)$ & The mapping that can transform $\rvx$ to $\rvh$ with $\rvg$\\
            $\epsilon$ & The error of the approximation \\
            ${\rm sigmoid}(\vx)$ & The Sigmoid function ${\rm sigmoid}(\cdot) = \frac{1}{1+e^{-\vx}}$ \\
            $n$ & The number of layers of the neural network for apporximation \\
            $\mW$ & The weight matrix \\
            $\vx$ & The input vector \\
            $\vb$ & The bias (in the weight matrix) \\
            $\lambda_0'$ & The minimum eigenvalue of the weight matrix $\mW$ \\
            $\MLP_b(\cdot)$ & The bijective MLP function \\
            ${\rm GC}(\cdot)$ & The graph convolution function with respect to KG $\mathcal{G}$ \\
            $\GCS_{\theta_1}$ & The GCS model parameterized by $\theta_1$ \\
            $\mathcal{L}$ & The loss function (objective) of the optimization \\
            $\mZ$ & The output of GCS, i.e., set of output entity representations \\
            $\rvz$ & The random variable of the output of GCS \\
            $\sup$ & The supremum value \\
            $T$ & A class of functions \\
            $\mathcal{F}$ & Any class of functions \\
            $\Omega$ & The domain of a function \\
            $T_{\theta_2}$ & A class of functions parameterized by $\theta_2$, i.e., neural networks \\
            $\mathbb{P}$ & The probability distribution \\
            $\mathbb{P}^{|\mathcal{V}|}$ & The empirical distribution with $|\mathcal{V}|$ samples \\
            ${\rm NN}_{\sigma}(\cdot|\theta')$ & The neural network with activation function $\sigma(\cdot)$ and parameterized by $\theta'$\\
            $|\mU|$ & The norm of matrix $\mU$ \\
            $\mA_n$ & The normalized adjacency matrix \\
            $\hat{\mX}$ & The ground-truth entity representations/node features \\
            $\hat{\mX}^{*}$ & The variable matrix \\
            $\textbf{Tr}(\cdot)$ & The trace of a matrix \\
            $\epsilon_1, \epsilon_2$ & The error bound of entity representations/node features and adjacency matrix \\
            $\gamma$ & The Lagrangian multiplier \\
            $p(t)$ & The characteristic polynomial for weight matrix $\mW$ \\
            $\text{det}(\cdot)$ & The determinant of a matrix \\
			\bottomrule
			\hline
		\end{tabular}
	}	
\end{table*}

\section{Relation Distribution}\label{appendix:relation}
We present the features of relation label distribution using two real KGs used for integration in ERNIE and K-Adapter, and we briefly illustrate that existing probe methods cannot support the relation label well.
\begin{figure}[!htbp]
    \centering
    \subfigure{\includegraphics[scale=.37]{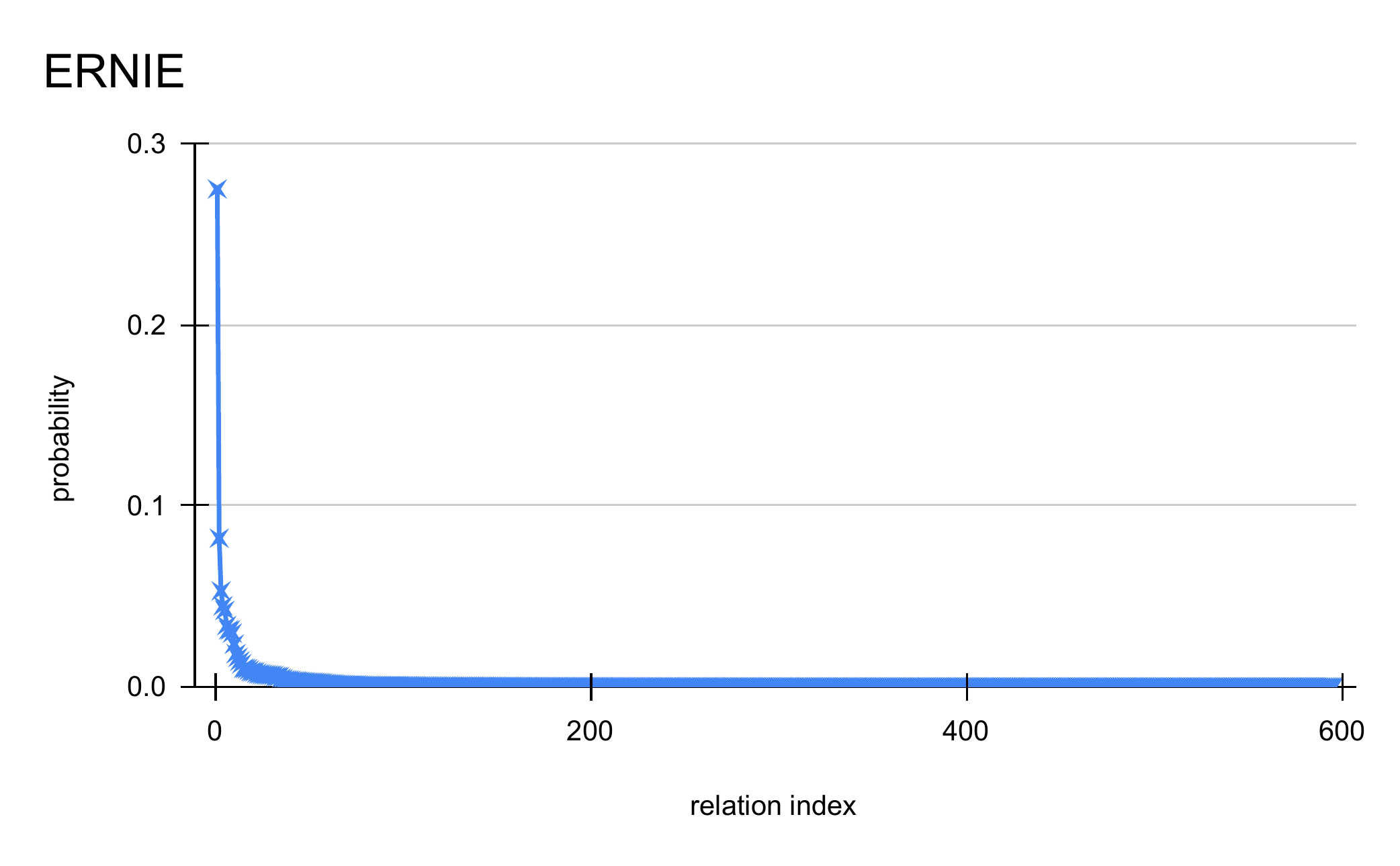}}
    \subfigure{\includegraphics[scale=.37]{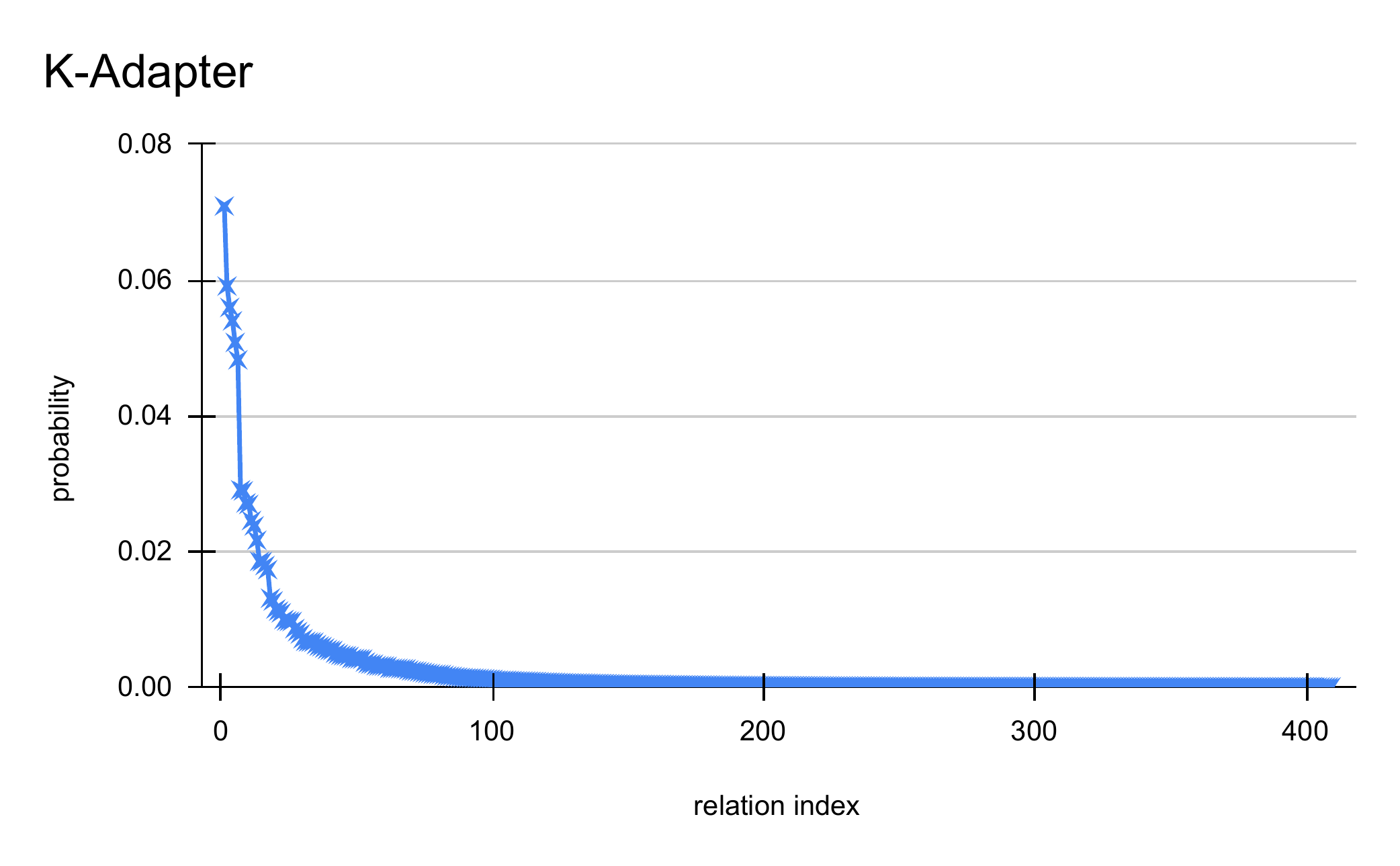}}
    \caption{The distribution of relations with respect to number in two KGs used in KI for ERNIE and K-Adapter.}
    \label{fig:relation_dist}
\end{figure}
\begin{itemize}
    \item \textbf{The number of relations is large.} The KG used in ERNIE has $594$ distinct relations and the KG used in K-Adapter has $410$ distinct relations. As for probes, if we differentiate them but use shared parameters such as the linear probe, the probe task is to predict whether the relation exist and if yes, which relation label it is. There is no wonder that using simple linear model cannot handle such difficult task with that large number of labels for classification. Regarding prompt probe, for each KG, the template for each relation should be manually created by human, and it is fairly costly.
    
    If we use different sets of parameters for different relations like in RGCN~\citep{rgcn_schlichtkrull18}. It is hard to implement hundreds or thousands of sets of parameters to analyze KI. For example, in our GCS model, we have to include attention mechanism for interpretation. The number of parameters for attention mechanism is hard to be scaled to $400-600$ times. Even if we can address this technical issue, using such as complex probe model for analysis is also problematic.
    
    \item \textbf{The distribution of relation number is highly imbalanced.} As shown in Figure~\ref{fig:relation_dist}, we can see that the distribution (i.e., PDF) of relations is very imbalanced. For ERNIE, $10\%$ relations account for $93\%$ edges, and $5$ relations (around $1\%$) account for $50\%$ edges. For K-Adapter, $10\%$ relations account for $78\%$ edges, and $5$ relations (around $1\%$) account for $29\%$ edges. Simply treating relations in different ways in interpretation could also provide problematic results. For example, in the linear probe, the simple linear model could not handle such highly imbalanced labels for classification.
\end{itemize}

\section{Formal Definitions of GFT and Graph Convolutions}\label{appendix:gft_def}
\noindent
\textbf{Formal Definition of GFT.}
Specifically, given a KG denoted as $\mathcal{G}$, let $\mA \in \mathbb{R}^{|\mathcal{V}|\times |\mathcal{V}|}$ be its symmetric adjacency matrix corresponding to $\mathcal{G}$. Let $\mL_n = \mI - \mD^{-1/2}\mA\mD^{-1/2}$ denote the normalized Laplacian matrix for $\mathcal{G}$, where $\mD$ denotes the degree matrix. We do the eigendecomposition for $\mL_n$ as $\mL_n = \mU \mLambda \mU^{T}$, where $\mU$ is the matrix of eigenvectors ordered by eigenvalues and $\mLambda = \text{diag}(\lambda_1, \lambda_2,...,\lambda_N)$ is the diagonal matrix of eigenvalues. Consider the node feature matrix as $\mX$. The GFT can be written as $\GFT(\mX) = \mU^{T}\mX$, and the inverse GFT can be written as $\RGFT(\GFT(\mX)) = \mU \GFT(\mX) = \mU \mU^{T} \mX = \mX$.

\noindent
\textbf{Formal Definition of Graph Convolutions.}
Graph convolutions~\citep{gc_npf_joan14} can be implemented by filters (i.e., kernels) as $g_{\Theta}$ in the graph spectral domain (i.e., KG space).
As the GFT of the convolution of $g_{\Theta}$ and $\mX$ is the pointwise product of their GFT~\citep{convolutiontheorem_newbold86}, the convolution can be written as 
\begin{equation}\small
    \begin{aligned}
        {\rm GC}(\mX) = g_{\Theta} \star \mX = \RGFT( g_{\Theta} \cdot  \GFT(\mX)).
    \end{aligned}\nonumber
\end{equation}
Regarding graph filters, \citet{gat_velickovic18} and \citet{agnn_thekumparampil18} introduce the attention mechanism to them, where the contribution of each edge to the convolution can be shown explicitly. Graph attention makes filters more powerful and convolutions more interpretable~\citep{gat_velickovic18,agnn_thekumparampil18,smoothness_guoji20,zheng-etal-2020-document}.

\section{Proof of Theorem~\ref{th_t1}}\label{appendix:theorem1}
\begin{proof}
    As aforementioned, the graph Fourier transformation $\GFT(\cdot)$ and its inverse transformation $\RGFT(\cdot)$ in terms of the KG $\mathcal{G}$ can be written as
    \begin{equation}\small
        \begin{aligned}
            \GFT(\mX) & = \mU^{T}\mX \\ 
            \RGFT(\GFT(\mX)) & = \mU \GFT(\mX) = \mU \mU^{T}\mX =\mX.
        \end{aligned}\nonumber
    \end{equation} 
    The second equation can be derived since $\mU$ is the set of eigenvectors of the normalized Laplacian matrix in terms of $\mathcal{G}$, which is orthogonal.

    According to the universal approximation theorem~\citep{uat_cybenko92}, in general, we can use one-layer neural networks (arbitrary width) with the sigmoid activation function to fit any functions. \citet{uat_depth_ilsang19} bound the approximation with both the width and depth, and supports more activation functions. Based on the conclusion of \citet{uat_depth_ilsang19}, we know that given a mapping $g'(\cdot)$, for any $\epsilon'>0$, there exists a neural network parameterized by $\theta'$ s.t.
    \begin{equation}\small
        \begin{aligned}
            |g'(\cdot)-{\rm NN}_{\sigma}(\cdot|\theta')| < \epsilon'.
        \end{aligned}\nonumber
    \end{equation} 
    Note that there are some constraints about the input and the model architecture, i.e., layer width. We leave out those details for simplicity since we only focus on the existence.

    Since $\rvh$ is obtained by integrating $\rvg$ into $\rvx$, we can simplify the mapping in the graph spectral space by researching on the transformation from $\GFT(\rvx)$ to $\GFT(\rvh)$.\footnote{In next proof, we illustrate that the linear transformation in the graph spectral space is graph convolution, which integrates the graph information into entities. Thus, in the graph spectral space, we do not need to regard $\rvg$ as an input. More formally description can be found in~\citet{uat_gnn_chen19,uat_gnn_keriven19}.} Assume the mapping satisfies $g'(\GFT(\rvx)) = \GFT(\rvh)$. Then we have 
    \begin{equation}\small
        \begin{aligned}
            |g'(\GFT(\rvx))-{\rm NN}_{\sigma}(\GFT(\rvx)|\theta')| < \epsilon'.
        \end{aligned}\nonumber
    \end{equation} 
    Consider that  we have  $f(\rvx, \rvg) = \rvh = \RGFT(g'(\GFT(\rvx)))$. If we assign $\epsilon'=\frac{\epsilon}{|\mU|}>0$, we have 
    \begin{equation}\small
        \begin{aligned}
            |\mU|\cdot|g'(\GFT(\rvx)) - {\rm NN}_{\sigma}(\GFT(\rvx)|\theta')|<\epsilon.
        \end{aligned}\nonumber
    \end{equation} 
    Since we know that 
        \begin{equation}\small
        \begin{aligned}
            \mU \cdot g'(\GFT(\rvx)) = \RGFT(g'(\GFT(\rvx))) = \rvh = f(\rvx, \rvg),
        \end{aligned}\nonumber
    \end{equation}
    we have
    \begin{equation}\small
        \begin{aligned}
            & |f(\rvx, \rvg) - \RGFT({\rm NN}(\GFT(\rvx)))| \\ 
            & <|\mU|\cdot|g'(\GFT(\rvx)) - {\rm NN}_{\sigma}(\GFT(\rvx)|\theta')|<\epsilon,
        \end{aligned}\nonumber
    \end{equation}
    where ${\rm NN}(\cdot)$ is parameterized by $\theta'$ with activation function $\sigma$ as ${\rm NN}_{\sigma}(\cdot|\theta')$. And without loss of generality, we assume it is composed of $n$ layers. 
\end{proof}

\section{Proof of Theorem~\ref{th_t2}}\label{appendix:theorem2}
According to the invariance property of MI~\citep{invariance_mi_kraskov04}, the introduction of bijective functions does not introduce any new information -- MI remains unchanged upon the introduction of bijective functions. We know that GFT and RGFT are both bijective (Appendix~\ref{appendix:proposition_simulation_1}). We show that nonlinear activation functions in a neural network (e.g., ${\rm sigmoid}(\cdot)$) are bijective as well (Appendix~\ref{appendix:proposition_simulation_1}). Thus, the MI change in the KI process can only happen in the linear function (Appendix~\ref{appendix:proposition_simulation_2}).
Based on the convolution theorem~\citep{convolutiontheorem_newbold86}, linear functions in graph spectral domain are graph convolution operations~\citep{spectrum_smoothness_sandryhaila14,gc_npf_joan14,gcn_thomas17} (Appendix~\ref{appendix:proposition_simulation_3}). 
Consider that the graph attention can show how the information flow on graphs during the convolution~\citep{zheng-etal-2020-document,smoothness_guoji20} (Appendix~\ref{appendix:proposition_simulation_4}). Thus, we can use graph convolutions in the transformation to interpret the KI process.

\begin{proof}
We present the proof with 4 steps below.

\subsection{Step 1}\label{appendix:proposition_simulation_1}
\textbf{$\GFT(\cdot)$, $\RGFT(\cdot)$, and ${\rm sigmoid}(\cdot)$ are bijective.} 
Given two entity representations $\vx_i$, $\vx_j$ and the matrix of eigenvectors of the KG as $\mU$, suppose that $\GFT(\vx_i) = \GFT(\vx_j)$. Then, we have 
\begin{equation}\small
    \begin{aligned}
        \mU^{T}\vx_i = \mU^{T}\vx_j.
    \end{aligned}\nonumber
\end{equation}
Since $\mU^{T}$ are set of eigenvectors and are by definition nonzero, we have 
\begin{equation}\small
    \begin{aligned}
        \vx_i=\vx_j.
    \end{aligned}\nonumber
\end{equation}
If $\vx_i=\vx_j$, it is easy to get $\GFT(\vx_i) = \GFT(\vx_j)$. Thus, graph Fourier transformation is bijective.

As for the nonlinear activation function, since we consider neural networks composed of MLP layers, the activation function is ${\rm sigmoid}(\cdot)$ function. It is easy to find that its inverse function is $f(\vy) = \ln(1-\frac{1}{\vy})$. Similarly, we can prove that it is bijective as well.

\subsection{Step 2}\label{appendix:proposition_simulation_2}
\textbf{Information gain and loss can only happen in the linear function in the graph spectral domain.} 
Without the loss of generality, we set the anchor random variable as $\rvg$. Same result can be derived using any other random variables.
Based on the invariance of MI~\citep{invariance_mi_kraskov04}, we have
\begin{equation}\small
\begin{aligned}
    \MI(\rvx; \rvg) & = \MI(\GFT(\rvx); \rvg), \\ 
    \MI(\rvx; \rvg) & = \MI(\RGFT(\rvx); \rvg), \\
    \MI(\rvx; \rvg) & = \MI({\rm sigmoid}(\rvx); \rvg).
\end{aligned}
\label{eq:proof:simulation}
\end{equation}
Since we know that 
\begin{equation}\small
\begin{aligned}
    \MI(\rvh; \rvg) - \MI(\rvx; \rvg) > 0,
\end{aligned}\nonumber
\end{equation}
and the neural network can well approximate the mapping, we have 
\begin{equation}\small
\begin{aligned}
    & \MI(\rvh; \rvg) - \MI(\rvx; \rvg) \\
    &\cong \MI(\RGFT({\rm NN}(\GFT(\rvx))); \rvg) - \MI(\rvx; \rvg) \\
    & = \MI({\rm NN}(\GFT(\rvx)); \rvg) - \MI(\GFT(\rvx); \rvg) >0.
\end{aligned}\nonumber
\end{equation}
If we write ${\rm NN}(\cdot)$ with $n$ MLP layers as $n\times \sigma({\rm Linear}(\cdot))$, we have
\begin{equation}\small
\begin{aligned}
    &\MI({\rm NN}(\GFT(\rvx)); \rvg) - \MI(\GFT(\rvx); \rvg) \\
    &= \MI(n\times \sigma({\rm Linear}(\GFT(\rvx)); \rvg) - \MI(\GFT(\rvx); \rvg) .
\end{aligned}\nonumber
\end{equation}
Recursively with equations~\ref{eq:proof:simulation}, it is easy to get that MI only changes in the ${\rm Linear}(\cdot)$ functions. And if we can show that linear function in the graph spectral domain is the graph convolution operation, we can then easily get that 
\begin{equation}\small
\begin{aligned}
    &\MI(n\times \sigma({\rm Linear}(\GFT(\rvx)); \rvg) \\
    & = \MI(n\times \MLP_b({\rm GC}(\MLP_b(\rvx)); \rvg).
\end{aligned}\nonumber
\end{equation}

\subsection{Step 3}\label{appendix:proposition_simulation_3}
\textbf{The linear function in the graph spectral domain is the graph convolution operation.}
Even if many existing works~\citep{spectrum_smoothness_sandryhaila14,gc_npf_joan14,gcn_thomas17} have provided clear descriptions, we simply re-illustrate it under the multi-channel setting. Consider the graph filter in~\citet{gc_npf_joan14} as an exmaple.

For a linear function $f(\vx) = \mW \times \vx$, its weight matrix $\mW \in \mathbb{R}^{F\times C}$ is parameterized by $\Theta \in \mathbb{R}^{F\times C}$. If the parameters are not shared for all nodes, the input $\mX \in \mathbb{R}^{|\mathcal{V}|\times C}$ can be rescaled in $\mathbb{R}^{|\mathcal{V}|\times C\times 1}$, and the weight matrix is $\mW \in \mathbb{R}^{|\mathcal{V}| \times F\times C}$ parameterized by $\Theta \in \mathbb{R}^{F\times C \times |\mathcal{V}|}$. The output of this linear function is mapped in $\mathbb{R}^{|\mathcal{V}|\times F}$.

Consider the signal in graph convolution, i.e., all $\vx$ in $\mX \in \mathbb{R}^{|\mathcal{V}|\times C}$. Since parameters are not shared~\citep{gc_npf_joan14}, for one graph filter, the parameters in $g_{\Theta}$ is in $\mathbb{R}^{C\times|\mathcal{V}|\times|\mathcal{V}|}$ that is parameterized by $\Theta \in \mathbb{R}^{C\times|\mathcal{V}|}$ with simple diagonalization. If we have $F$ different graph filters for the convolution, $g_{\Theta}$ is in $\mathbb{R}^{F\times C\times|\mathcal{V}|\times|\mathcal{V}|}$ that is parameterized by $\Theta \in \mathbb{R}^{F\times C\times|\mathcal{V}|}$. Here, the graph Fourier transformation of $\mX$ is $\GFT(\mX) \in \mathbb{R}^{|\mathcal{V}|\times C}$, which can be rescaled in $\mathbb{R}^{1\times |\mathcal{V}| \times C \times 1}$ with simple diagonalization. The output is in $\mathbb{R}^{F\times |\mathcal{V}|\times |\mathcal{V}|\times 1}$. Note that since the parameters in the graph filter is diagonalized, we can rescale the output in $\mathbb{R}^{|\mathcal{V}|\times F}$. 

If we regard the weight matrix $\mW$ as the parameters in the graph filter $g_{\Theta}$, the input matrix $\mX$ as the signal, obviously, the linear function in the graph spectral space is the graph convolution operation.

\subsection{*Step 4}\label{appendix:proposition_simulation_4}
\textbf{Graph attention can show how information flows on the graph.} 
Graph attention works as denoising information from neighbors, since it can adaptively learn the optimal weights (i.e., attention coefficients) for different neighbors. If the node features of a neighbor contain much useless information for the center node, the learned weight should be small to denoise that information. It can show how information (i.e., node features) flows among nodes on the KG structure.

Consider a graph signal denoising problem that we aim to extract the ground-truth node features $\hat{\mX}$ and edge weights $\hat{\mA}_n$ from a graph $\mathcal{G}=(\mathcal{V}, \mathcal{E}, \mA_n)$ with noise in both node features $\mX$ and edge weights $\mA_n$. Here, $\mA_n$ is the normalized adjacency matrix $\mA_n = \mD^{-1/2}\mA\mD^{-1/2}$. To this end, we formulate the optimization problem under the assumption that the ground-truth node features $\hat{\mX}$ are smooth w.r.t the ground-truth adjacency matrix $\hat{\mA}_n$ and the noise in the graph can be upper-bounded:
\begin{equation}\label{app-eq1}\small
	\begin{aligned}
		\hat{\mX}^*, \hat{\mA}_n^* = {} & \underset{\hat{\mX}, \hat{\mA}_n}{\text{argmin}}\mathrm{Tr}\left(\hat{\mX}\hat{\bm{L}}_n^{T}\hat{\mX}\right) \\
		\text{s.t. } {} & \|\hat{\mX} - \mX\|_2^2 \leq \epsilon_1, \\
		& \|\hat{\mA}_n - \mA_n\|_2^2 \leq \epsilon_2,
	\end{aligned}
\end{equation}
where $\hat{\mL} = \mI - \hat{\mA}$, $\epsilon_1, \epsilon_2 \in \mathbb{R}$, are the level of noise in node features and edge weights, respectively. $\mathrm{Tr}(\cdot)$ indicates the trace of a matrix. By Lagrange multipliers methods, we can obtain the solution as following:
\begin{equation}\label{app-eq2}\small
	\hat{\mX}^* = \frac{\gamma}{1+\gamma}\left(\mI - \frac{1}{1+\gamma}\hat{\mA}_n^*\right),
\end{equation}
\begin{equation}\label{app-eq3}\small
	\hat{\mA}_n^* = \mA_n + \sqrt{\epsilon_2}\frac{\hat{\mX}^*\hat{\mX}^{*\top}}{\|\hat{\mX}\|_2^2},
\end{equation}
where $\gamma > 0$ is the Lagrangian multiplier. Note that the attention coefficients of GAT~\citep{gat_velickovic18} and AGNN~\citep{agnn_thekumparampil18} are obtained by (without less of generality, we show the results in the first-layer) \eqref{app-eq4} and \eqref{app-eq5}, respectively:
\begin{equation}\label{app-eq4}\small
	a_{i,j} = \text{softmax}\left(\text{LReLU}\left(\mathbf{a}^\top\left[\bm{WX}_i\|\bm{WX}_j\right]\right)_{j \in \mathcal{N}_i \cup \{i\}}\right),
\end{equation}
\begin{equation}\label{app-eq5}\small
	a_{i,j} = \text{softmax}\left(\left[\beta\frac{\bm{H}_i^\top\bm{H}_j}{\|\bm{H}_i\|\|\bm{H}_j\|}\right]_{j \in \mathcal{N}_i \cup \{i\}}\right),
\end{equation}
where $\text{LReLU}$ is the leakyReLU; $\bm{H} = \text{ReLU}(\bm{XW})$, $\mathbf{a}$, $\bm{W}$ in \eqref{app-eq4}, and $\beta$, $\bm{W}$ in \eqref{app-eq5} are learnable parameters. The attention coefficents of GAT and AGNN are then used as the weights of aggregating the neighbohood information of nodes. As we can see that \eqref{app-eq3}, \eqref{app-eq4}, and \eqref{app-eq5} are in a form of measuring the similarity between paired node features. Similar to the denoised edge weights obtained in \eqref{app-eq3}, the attention coefficents (i.e. the aggregation weights) between a node and its neighborhoods are proportional to the similarity of their node embeddings. Therefore, the attention coefficients of GAT and AGNN can be regarded as the results of denoised weights on the existing edges in a graph, i.e., the graph attentions are implicitly denoising the edge weights.

In general case, graph attention functions as denoising edge weights. The input is noisy representations and the output is the groundtruth. Attention coefficients show how much distortion is corrected during the convolution operation. For example, if the input representations are also groundtruth, there is no need to fetch information from neighbors to get output. And edge weights will be reduced to $0$, i.e., attention coefficients on edges are calculated as $0$. If the input representations are very noisy, i.e., much noise are removed, attention coefficients on edges should be large to restore the groundtruth signal.
Therefore, in the KI scenario, we can use attention coefficients in graph attention in graph convolution layer to interpret the KI process based on the information flow. As for the CR and CF, equally, we can use the attention coefficients on the self-loop edges for interpretation, such as how much original information is remembered/forgotten.

\end{proof}

\section{Bijective MLP} \label{appendix:bijection}
\begin{theorem} \label{th_bj}
	Give an MLP layer denoted as $\MLP(\vx) = {\rm sigmoid}(\mW\vx + \vb).$ If $\mW$ is a square matrix, there exist a constant $\lambda_0'>0$ that for any $0<\epsilon<\lambda_0'$, the function below is bijective: 
	\begin{equation}\label{eq:mlpn}\small
	\MLP_n(\vx) = {\rm sigmoid}((\mW-\epsilon\mI) \vx + \vb).
	\end{equation}
\end{theorem}
\begin{proof}
We first prove that two bijective function compositions are still bijective. Then, we prove that adding a small noise on MLP weight matrix can make it bijective.

Give two MLP function $f_1(\cdot)$ and $f_2(\cdot)$. Suppose they are injective and suppose $f_1(f_2(\vx)) = f_1(f_2(\vy))$. Since we know that $f_1(\cdot)$ is injective, we have $f_2(\vx) = f_2(\vy)$. Similarly, since $f_2(\cdot)$ is injective, we have $\vx=\vy$. Thus $f_1(f_2(\cdot))$ is injective. Suppose $f_1(\cdot)$ and $f_2(\cdot)$ are surjective and $\vz \in C$. Since we know that $f_1(\cdot)$ is surjective, there exists a set of $\vy \in B$ with $f_1(\vy)=\vz$. Similarly, since $f_2(\cdot)$ is surjective, there exists a set of $\vx \in A$ with $f_2(\vx) = \vy$. Then, we have $\vz = f_1(f_2(\vx))$ and so $\vz$ is onto $f_1(f_2(\cdot))$. Thus, $f_1(f_2(\cdot))$ is surjective. Therefore, if $f_1(\cdot)$ and $f_2(\cdot)$ are bijective, $f_1(f_2(\cdot))$ is also bijective.

To prove that the special MLP is bijective, consider an MLP function as 
\begin{equation}\small
    \begin{aligned}
        \MLP(\vx) = \sigma(\mW\vx + \vb),
    \end{aligned}\nonumber
\end{equation}
where $\mW \in \mathbb{R}^{C\times C}$ is the weight matrix and $\vb \in \mathbb{R}^{C}$ is the bias. Let 
\begin{equation}\small
    \begin{aligned}
        p(t) = \prod_{i=1}^{C} (\lambda_{i}' -t) 
    \end{aligned}\nonumber
\end{equation}
be the characteristic polynomial for weight matrix $\mW$. Here $\lambda_{i}'$ are eigenvalues of matrix $\mW$. Without loss of generality, let $|\lambda_{0}'| = \min_{i}{|\lambda_{i}'|}$. Then, we know that for any constant $0 < \epsilon < |\lambda_{0}'|$, we have 
\begin{equation}\small
    \begin{aligned}
        \text{det}(\mW-\epsilon\mI)=p(\epsilon) \neq 0.
    \end{aligned}\nonumber
\end{equation}
Thus, if the perturbation $\epsilon$ is small enough, the perturbed matrix $\mW'=\mW-\epsilon\mI$ is nonsingular. Consider the fact that the nonlinear activation function $\sigma(\cdot)$ is ${\rm sigmoid}(\cdot)$ function, which is bijective. Therefore, the special MLP function $\MLP_n(\cdot)$ is bijective.

Note that in practice, we use the floating-point arithmetic. Consider the float accuracy. Small errors from the float approximation can be regarded as the constant $\epsilon$, and in most cases, it satisfies the assumption $0 < \epsilon < |\lambda_{0}'|$. Thus, we can regard MLPs with square weight matrices in practice as bijective functions.

\end{proof}

\section{Implementation Details}
\subsection{GCS}\label{appendix:gcs_param}
In practice, GCS is composed of $3$ layers: bijective MLP layer, graph convolutional layer, and another bijective MLP layer. As for bijective MLP layers, since weight matrices in them are square matrices, the dimension would remain unchanged: $768$ for ERNIE and $1024$ for K-Adapter. The nonlinear activation functions are set as ${\rm ELU}(\cdot)$ function, which is also bijective. The learning rate is set as $1e^{-3}$, and the dropout rate of the two MLP layers is $0.2$. 

Regarding the graph attention, to make sure interpretation results are stable, we apply multi-head attention mechanism, where the number of attention head is set as $8$. Entity representations are first embedded into a space with the dimension as $64$. Then, the embedded representations are used to calculate the attention coefficients. Note that since the purpose is to simulate and interpret the KI process, we do not split datasets for KI. Considering that GCS model is very simple for large KGs, overfitting is unlikely to happen. Thus, we optimize GCS for the whole datasets. Specifically, for K-Adapter, the whole KG is used for optimization, and results are used for interpretation. And for ERNIE, since the KG is very large, we sample a small subgraph with $1,344,393$ entities and $3,240,272$ triples for optimization (see Table~\ref{tab:statistics}), and then implement the optimized GCS on the whole KG for interpretation.

The objective function of optimizing GCS can be reconstruction loss minimization or MI maximization. In this paper, we all select MI maximization as the objective. Note that users can use other objectives such as the reconstruction loss minimization. Regarding the MI maximization, we optimize MI (equation~\ref{eq:gcsobj}) by maximizing the compression lemma lower bound~\citep{mi_cllb_banerjee06} as in~\citet{mine_belghazi18}. The inputs of GCS are $\mX$, and let the output be denoted by $\mZ$. We can regard $\mZ$ and $\mH$ as empirical samples of random variables $\rvz$ and $\rvh$. Thus, we have:
\begin{equation}\label{eq:gcs_obj_lb}\small
\MI(\rvz; \rvh) \geq \sup_{T\in\mathcal{F}}\mathbb{E}_{\mathbb{P}_{\rvz \rh}}[T] - \log(\mathbb{E}_{\mathbb{P}_{\rvz} \otimes \mathbb{P}_{\rvh}}[e^{T}]).
\end{equation}
Here, $\mathcal{F}$ can be any class of functions $T: \Omega \rightarrow \mathbb{R}$ satisfying certain integrability constraints~\citep{mine_belghazi18}. $\mathbb{P}_{\rvz \rvh}$ represents the joint distribution of $\rvz$ and $\rvh$, and $\mathbb{P}_{\rvz} \otimes \mathbb{P}_{\rvh}$ represents the product of their marginal distributions. In practice, we let $\mathcal{F} = \{T_{\theta_2} \}$ be the set of functions parameterized by a neural network, and optimize it by stochastic gradient descent. Then, the objective function can be rephrased as 
\begin{equation}\label{eq:gcs_obj_real}\small
\begin{aligned}
&\max_{\theta_1, \theta_2}{\Big( \mathbb{E}_{\mathbb{P}^{|\mathcal{V}|}_{\rvz, \rvh}}[T_{\theta_2}] - \log\big( \mathbb{E}_{\mathbb{P}^{|\mathcal{V}|}_{\rvz} \otimes \mathbb{P}^{|\mathcal{V}|}_{\rvh} }[e^{T_{\theta_2}}] \big) \Big)}, \\
& \text{ where } \rvz = \GCS_{\theta_1}(\rvx).
\end{aligned}
\end{equation}
In equation~\ref{eq:gcs_obj_real}, $\mathbb{P}^{|\mathcal{V}|}_{\rvz}$ represents the empirical distribution of $\rvz$, i.e., $\mZ$. If the KG is very large, we can optimize the network by sampling a small subgraph of the KG. In practice, we simply add two extra MLPs layers to GCS for MI maximization as~\citep{mine_belghazi18}. The added two MLP layers may not be bijective, where the dimension would be first reduced to $64$, then to $1$ for MI maximization. The nonlinear activation functions are all set as ${\rm ELU}(\cdot)$ function, which is also bijective. 

For interpretation, we use the attention coefficients on edges and self-loops to analyze the KI in terms of triples and entities. Different from~\citet{edgemasking_schlichtkrull21} that specially designs a discrete function to mask edges that are not important, we simply introduce a temperature hyperparameter $t$ and set it as $t=0.1$ to make the attention coefficient distribution hard.\footnote{Note that the principle of hyperparameter selection is to maximize the MI, i.e., objective function. Users may select appropriate hyperparameters depending on the situation.} Thus, knowledge can be well clustered into learned and unlearned.

\subsection{ERNIE and K-Adapter}\label{appendix:lm_param}
\textbf{KI.}
To ensure that the experiment settings are fair, we set hyperparameters as the default values. For \textbf{K-Adapter}, the code and hyperparameters for KI that we use are from the official projects\footnote{\href{https://github.com/microsoft/K-Adapter}{https://github.com/microsoft/K-Adapter}} published by the authors~\citep{wang-etal-2021-k}. The only two differences are that: we use PyTorch \textit{float 32} instead of \textit{float 16} since BERT and RoBERTa that we use are \textit{float32}, and we use $4$ NVIDIA Tesla V100 GPUs for KI training. 
For \textbf{ERNIE}, settings are the same. All hyperparameters for KI are set as their default values.\footnote{\href{https://github.com/thunlp/ERNIE}{https://github.com/thunlp/ERNIE}} Similarly, \textit{float 16} of PyTorch is changed to \textit{float 32}, and we do the integration with $4$ NVIDIA Tesla V100 GPUs. Note that the dataset that ERNIE used for KI is Wikipedia, since the code is to fetch latest version of it, the data that we use could be slightly different. Therefore, for both ERNIE and K-Adapter, to ensure the fairness, we reproduce their KI, and report the results of reproduced models instead of results provided in their papers.

\noindent
\textbf{Finetuning.}
As for the downstream tasks, all the hyperparameters are consistent with the official project: either they are given in the project or in the README. In the same way, \textit{float 32} and $4$ NVIDIA Tesla V100 GPUs are chosen to make sure that the comparison is fair. Note that for K-Adapter and ERNIE, the best performance for different datasets is achieved in different settings. For example, the best performance for K-Adapter on the OpenEntity dataset is achieved with single GPU, but on the FIGER dataset is achieved with four GPUs. Since we focus on the relative performance and the fairness of the comparison, we run finetuning on $4$ NVIDIA Tesla V100 GPUs for all downstream tasks and all LMs (as well as BERT and RoBERTa).

\section{Additional Statistics}\label{appendix:additional_statistics}
\begin{table*}[!htbp]
	\caption{Statistics of T-REx-rc and Wikidata. The datasets that K-Adapter and ERNIE use are T-REx-rc and Wikidata.}\smallskip
	\label{tab:statistics}
	\centering
	\resizebox{1.8\columnwidth}{!}{
		\smallskip\begin{tabular}{c|ccccc}
			\toprule
			{\diagbox{Datasets}{Statistics}} & \# of entities & \# of triples & \# of aligned sentences & \# of entities (optimization) &  \# of triples (optimization) \\
			\hline
            T-REx-rc & 781,275 & 1,282,504 & 5,565,478 & - & - \\
            Wikidata & 3,275,534 & 12,849,311 & - & 1,344,393 & 3,240,272 \\
			\bottomrule
			\hline
		\end{tabular}
	}	
\end{table*}

\begin{table*}[!htbp]
	\caption{Drop statistics for the Integration Experiment.}\smallskip
	\label{tab:drop}
	\centering
	\resizebox{1.8\columnwidth}{!}{
		\smallskip\begin{tabular}{c|ccc}
			\toprule
			\diagbox{Datasets}{Statistics} & Percentage of integrated entities & Percentage of integrated triples & \# of aligned sentences/entity embeddings (integrated knowledge) \\
			\midrule
            T-REx-rc & - & 28.86\% & 561,687 out of 5,565,478 \\
            Wikidata & 61.72\% & - & 2,240,260 out of 3,275,534\\
			\bottomrule
			\hline
		\end{tabular}
	}	
\end{table*}

\begin{table*}[!htbp]
	\caption{Performance change of K-Adapter and ERNIE on the OpenEntity dataset with different test sets.}
	\smallskip
	\label{tab:verify_output}
	\centering
	\resizebox{1\columnwidth}{!}{
		\smallskip\begin{tabular}{c|cccc}
			\toprule
			\multirow{2}*{Model (Test set)} & \multicolumn{4}{c}{OpenEntity}  \\
			\cline{2-5}
			~ & Left test set & P & R & $\Delta$F1-Micro \\
			\midrule
			K-Adapter (w/o IE) & 37.44\% & {\color{green} $-$} 0.33 & {\color{green} $-$} 0.37 & {\color{green} $-$} 0.35 \\
            K-Adapter (w/o UE) & 64.46\% & {\color{green} $-$} 0.18 & {\color{red} $+$} 1.12 & {\color{red} $+$} 0.47 \\
            \hline
            ERNIE (w/o IE) & 27.28\% & {\color{green} $-$} 18.20 & {\color{green} $-$} 25.14 & {\color{green} $-$} 22.67 \\
            ERNIE (w/o UE) & 66.87\% & {\color{green} $-$} 0.31 & {\color{red} $+$} 3.08 & {\color{red} $+$} 1.57 \\
			\bottomrule
			\hline
		\end{tabular}
	}	
\end{table*}

\begin{figure*}[!htbp]
	\centering
	\subfigure{\includegraphics[scale=.135]{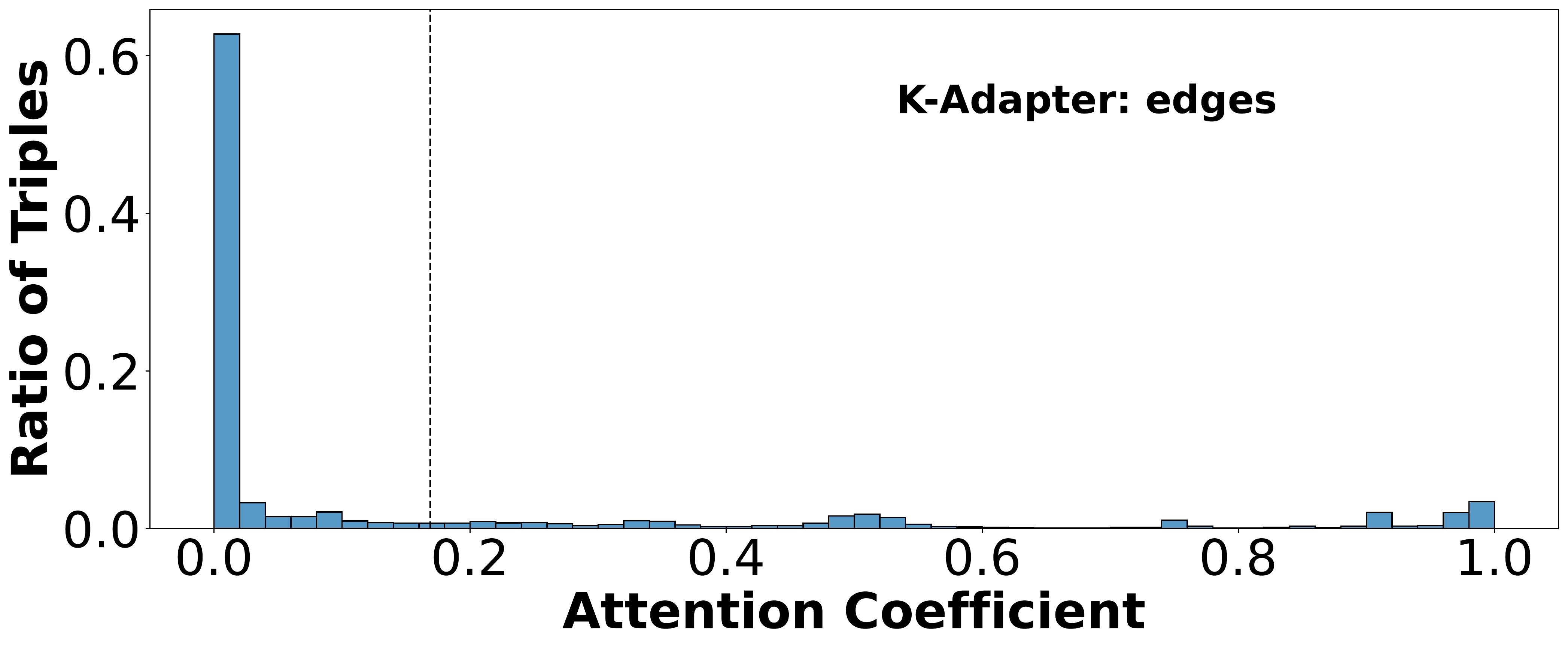}}
	\vspace{-.28cm}
	\subfigure{\includegraphics[scale=.135]{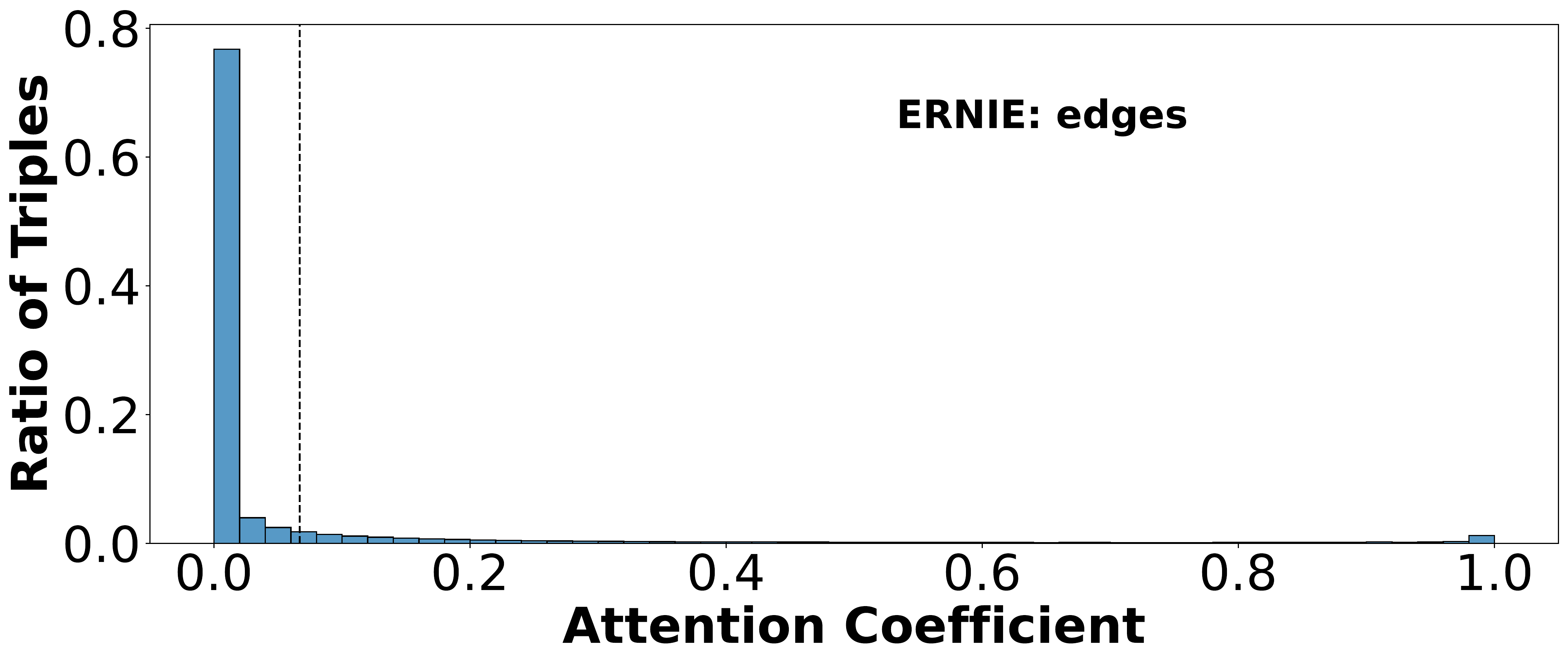}}
	\vspace{-.1cm}
	\subfigure{\includegraphics[scale=.135]{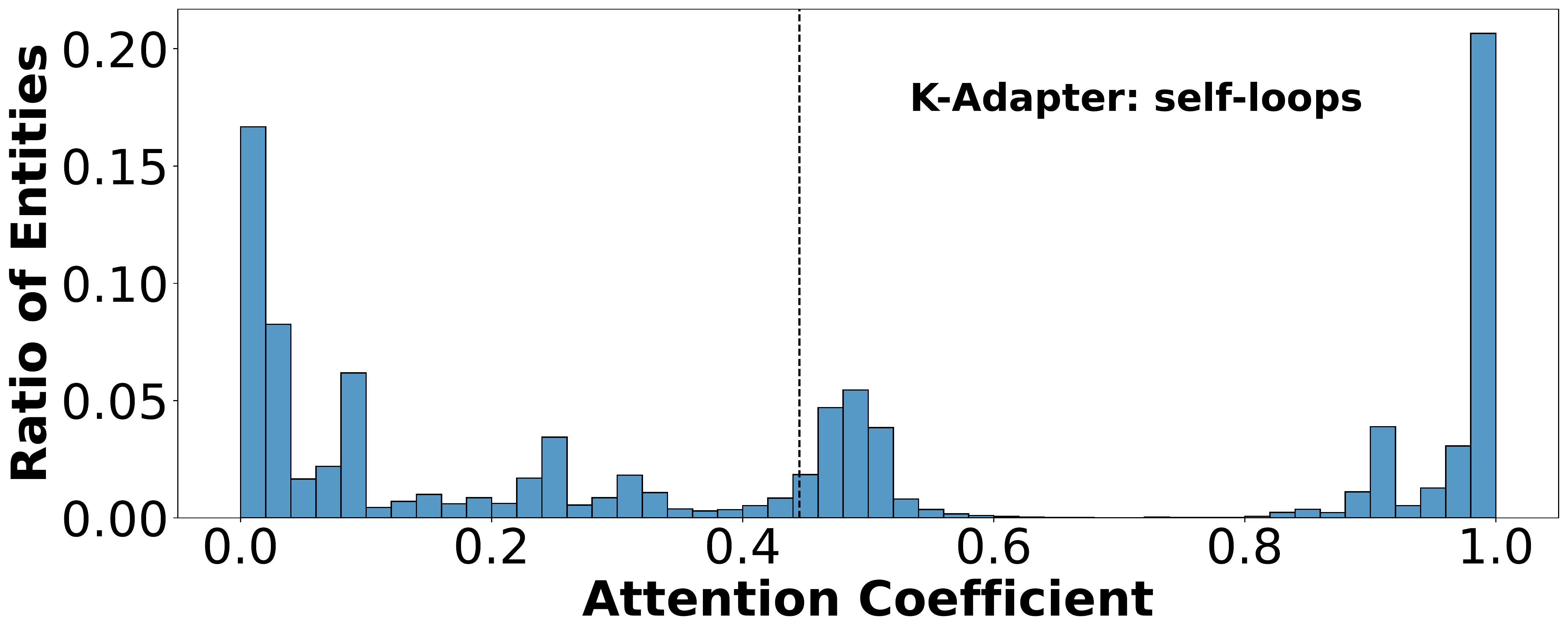}}
	\vspace{-.1cm}
	\subfigure{\includegraphics[scale=.135]{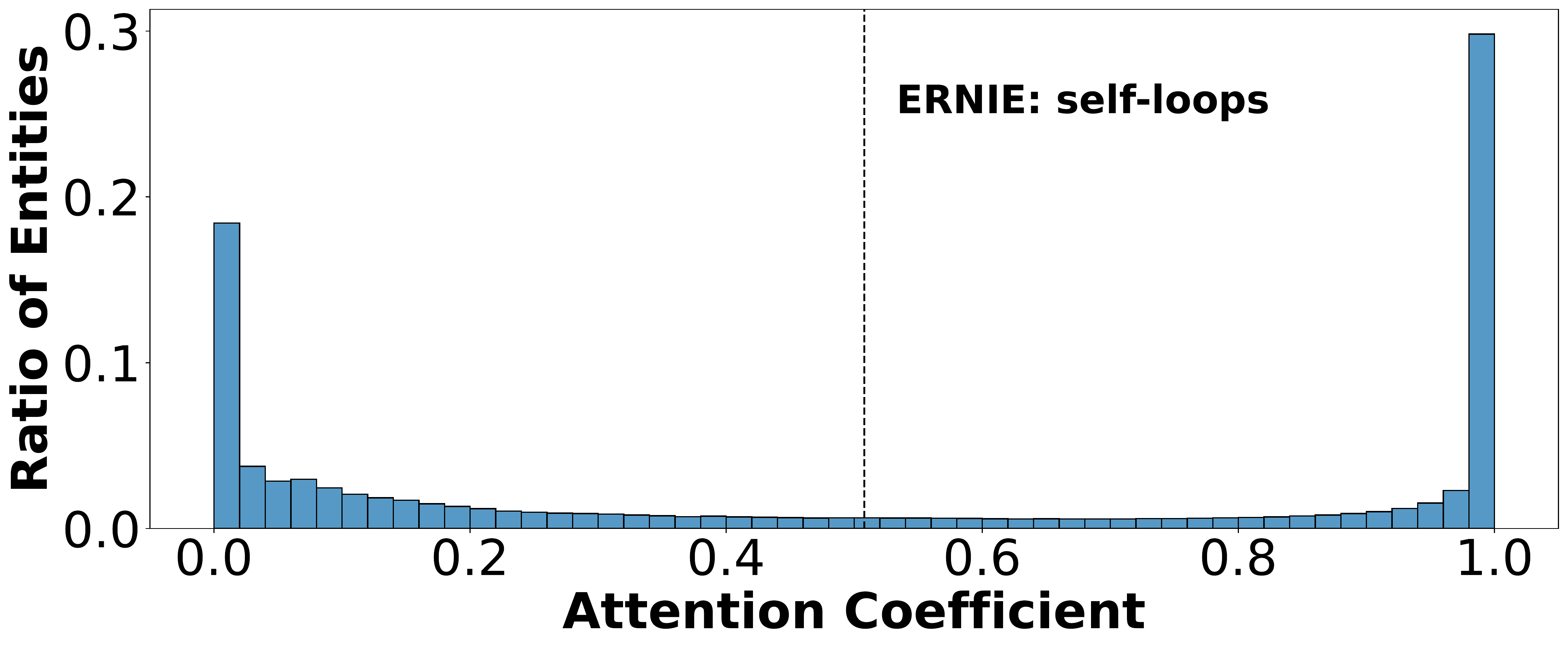}}
	\caption{The attention coefficient distributions of edges and self-loops for K-Adapter and ERNIE. The histogram shows the empirical distributions (i.e., frequency), and the blue curves are the Gaussian kernel density estimate. The black dashed vertical lines indicate the average values.}
	\label{fig:all_dist}
\end{figure*}

\begin{figure*}[!htbp]
    \centering
    \includegraphics[scale=.42]{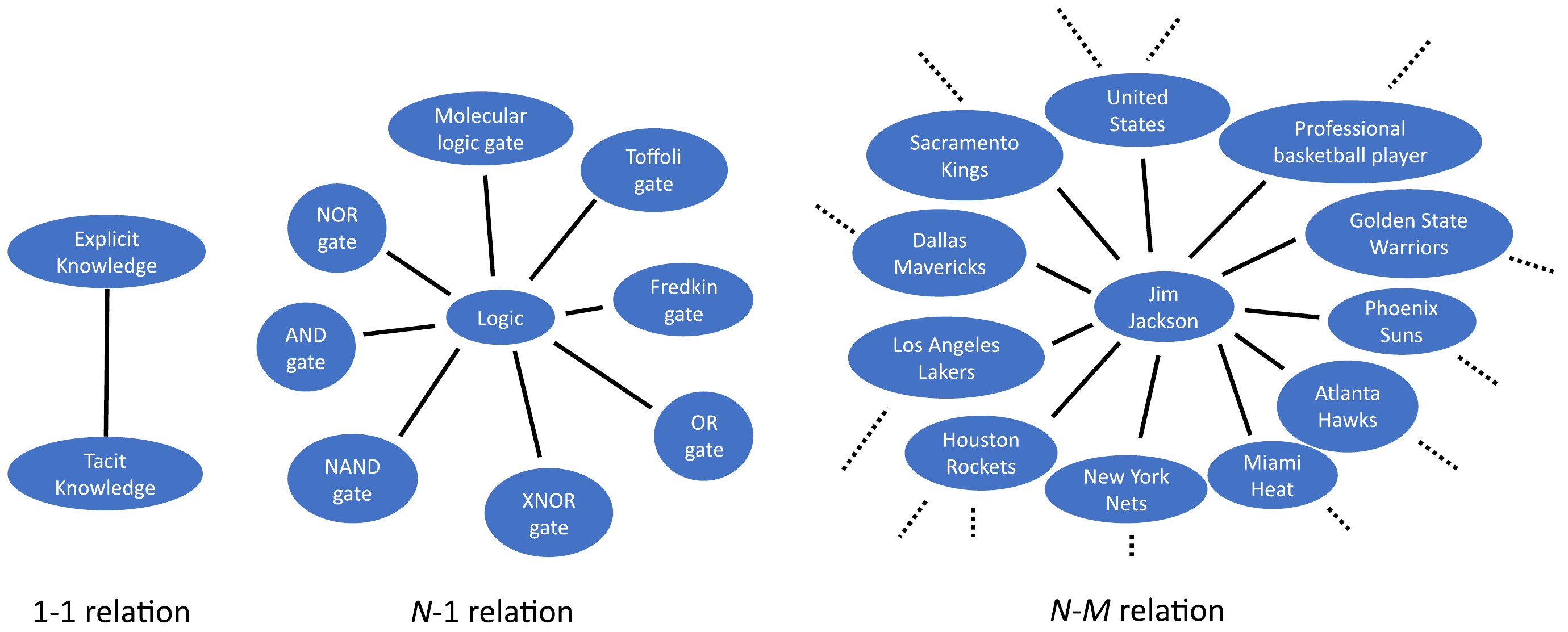}
    \caption{An example of different relations. The black solid lines are relations. We can find that nodes connect to multiple neighbors (e.g., "Logic") are more common than those leaf node (e.g., "OR gate").}
    \label{fig:kg_relation}
\end{figure*}

\begin{table*}[!htbp]
	\caption{Analysis of KI interpretation results for K-Adapter and ERNIE in terms of different types of relations (topology feature). The percentages of integrated entities/triples, as well as of CR and CF entities for each type of relations are presented.}
	\smallskip
	\label{tab:specific}
	\centering
	\resizebox{1.5\columnwidth}{!}{
		\smallskip\begin{tabular}{c|cccc}
			\toprule
			\multirow{2}{*}{\diagbox{Statistics}{Model}} & \multicolumn{4}{c}{K-Adapter on T-REx-rc} \\
			\cline{2-5}
			~ & $1-1$ relation & $N-1$ relation & $N-M$ relation & Total  \\
			\midrule
            \# of triples & 21,690 & 813,674 & 1,729,644 & 2,565,008 \\
            Integrated triple percentage & 58.89\% & 38.39\% & 24.00\% & 28.86\% \\
            \hline
            \# of connected entities & 21,690 & 406,837 & 352,748 & 781,275 \\
            CR entity percentage & 41.11\% & 31.72\% & 26.02\% & 29.41\%\\
            CF entity percentage & 26.40\% & 30.29\% & 40.89\% & 34.97\%\\
            \hline
            \hline
			\multirow{2}{*}{\diagbox{Statistics}{Model}} & \multicolumn{4}{c}{ERNIE on Wikidata} \\
			\cline{2-5}
			~ & $1-1$ relation & $N-1$ relation & $N-M$ relation & Total  \\
			\midrule
            \# of connected entities & 1,799 & 529,186 & 2,744,549 & 3,275,534 \\
            Integrated entity percentage & 70.65\% & 42.86\% & 73.33\% & 68.39\% \\
            \hline
            CR entity percentage & 29.41\% & 56.07\% & 26.67\% & 38.28\%\\
            CF entity percentage & 23.18\% & 8.65\% & 37.10\% & 32.49\%\\
			\bottomrule
			\hline
		\end{tabular}
	}	
\end{table*}

\begin{table*}[!htbp]
	\caption{The number of aligned sentences for relations.}\smallskip
	\label{tab:label_num}
	\centering
	\resizebox{1.2\columnwidth}{!}{
		\smallskip\begin{tabular}{c|c}
			\toprule
			\diagbox{Relation label}{Statistics} & \multicolumn{1}{c}{\# of triples} \\
			\midrule
            Place of birth & 134,976 \\
            Part of & 134,999 \\
            \hline
            Date of death & 135,190 \\
            Date of birth & 135,169 \\
            \hline
            Located in the administrative territorial entity & 135,055 \\
            Country & 135,147 \\
            \hline
            Total & 5,565,478 \\
			\bottomrule
			\hline
		\end{tabular}
	}	
\end{table*}


\end{document}